\documentclass{article} 
\usepackage{iclr2027_conference,times}

\usepackage{amsmath,amsfonts,bm}

\def\eqref#1{equation~\ref{#1}}

\def\1{\bm{1}}

\DeclareMathAlphabet{\mathsfit}{\encodingdefault}{\sfdefault}{m}{sl}
\SetMathAlphabet{\mathsfit}{bold}{\encodingdefault}{\sfdefault}{bx}{n}

\usepackage{hyperref}
\usepackage{url}

\usepackage{float}
\usepackage{amsmath}
\usepackage{amsthm}

\usepackage{amssymb}
\usepackage{algorithm}
\usepackage{algpseudocode}

\algrenewcommand\algorithmicrequire{\textbf{Require:}}
\algrenewcommand\algorithmicensure{\textbf{Return:}}

\usepackage{graphicx}
\usepackage{multirow}
\usepackage{makecell}
\usepackage{color}
\usepackage{subfigure}
\usepackage{xspace}
\usepackage{booktabs}
\usepackage{caption}

\newcommand{\ProposedM}{\textsc{GPARA}\xspace}

\theoremstyle{plain}
\newtheorem{prop}{\protect\propositionname}
\newtheorem{cor}{\protect\corollaryname}
\newtheorem{lem}{\protect\lemmaname}
\providecommand{\corollaryname}{Corollary}
\providecommand{\lemmaname}{Lemma}
\providecommand{\propositionname}{Proposition}

\title{GPARA: Graph-Posterior-Aligned Refinement and Active Acquisition for Grounding Diffusion Priors }

\newcommand{\affmark}[1]{%
  \textsuperscript{\normalfont #1}%
}
\newcommand{\authorsep}{,\nobreak\hspace{0.25em}}

\author{
\begin{minipage}{\dimexpr\textwidth-2\tabcolsep\relax}
\centering
Wangqian Chen\affmark{1}\authorsep
Hao Wang\affmark{2}\authorsep
Yumeng Zhang\affmark{1}\authorsep
Jiajia Guo\affmark{1}\authorsep
Junting Chen\affmark{2}\authorsep
Jun Zhang\affmark{1} \\
\vspace{0.35em}
{\normalfont
\textsuperscript{1}The Hong Kong University of Science and Technology}\\
{\normalfont
\textsuperscript{2}The Chinese University of Hong Kong, Shenzhen}\\[0.15em]
\end{minipage}
}

\hypersetup{hidelinks}

\iclrfinalcopy

\begin{document}

\maketitle

\fancyhead{}
\renewcommand{\headrulewidth}{0pt}

\begin{abstract}
Active grounding of a frozen diffusion prior requires jointly determining
where new measurements should be taken and how they should be used
to refine the current reconstruction. Posterior-ensemble-based methods
can estimate acquisition utility from generated samples, but require
repeated ensemble generation as observations accumulate and capture
posterior geometry only through empirical statistics. This paper proposes
GPARA, which learns a context-dependent graph surrogate over diffusion
prediction residuals, inducing an explicitly reusable posterior response
operator that propagates measurement innovations to unobserved variables
and evaluates candidate measurements through weighted posterior-risk
reduction. Under the matched surrogate, we show that the same response
operator also determines expected one-step acquisition benefit and
yields an analytic ranking consistent with expected reconstruction
improvement. A bounded learned residual calibrates the analytic utility
to account for surrogate mismatch, while a small prior ensemble is
generated once and reconditioned to update risk weights without repeated
diffusion posterior sampling during acquisition. Experiments on two
reconstruction tasks spanning physical field and computer vision
show consistent improvements in refinement and active acquisition
over the evaluated baselines. Ablations further support the complementary
roles of step-wise graph refinement, adaptive risk weighting, and
analytically anchored calibration.
\end{abstract}

\section{Introduction}\label{sec:introduction}

Diffusion models provide expressive priors for recovering high-dimensional
signals, images, and physical fields from heterogeneous information
\citep{Graiko2022,song_scoresde_2021,rombach_latentdiffusion_2022,FeiLyu2023}.
Given auxiliary context $\boldsymbol{c}$, such as image, feature
representations, geometry layout, or physical metadata, the learned
prior $p_{\theta}(\boldsymbol{x}\mid\boldsymbol{c})$ can generate
a plausible target $\boldsymbol{x}$. This approach has been applied
to imaging, remote sensing, field reconstruction, and geometric perception
\citep{wu_seesr_2024,khanna2024,li_palsb_2025,KeObu2024}. However,
conditional plausibility does not ensure agreement with the real target
realization. Information absent from the context can lead to instance-specific
discrepancies in the predicted structure and values.

Direct target measurements provide complementary evidence about the
current realization. Existing inverse solvers incorporate observations
through spectral or null-space decomposition \citep{kawar_ddrm_2022,song_pigdm_2023,wang_ddnm_2023},
likelihood guidance \citep{chung_dps_2023,aali_ambientdps_2025},
variational or moment-based posterior approximations \citep{mardani_reddiff_2024,boys_tmpd_2024,rout_stsl_2024},
data-consistency \citep{zhu_diffpir_2023,zirvi_diffstategrad_2025,song_resample_2024},
filtering \citep{dou_fps_2024}, or posterior sampling \citep{zhang_daps_2025,wu_pnpdm_2024}.
These methods primarily address reconstruction under a prescribed
observation operator.

Since measurement acquisition incurs sensing energy, bandwidth, or
other resources, the measurements are often sparse. The problem is
thus not only how to reconstruct from a prescribed observation set,
but also which observations are the most useful to acquire. We refer
to this setting as \emph{active grounding}: actively acquiring a limited
number of measurements to refine a frozen diffusion prior. This requires
evaluating how each candidate measurement would affect the unresolved
target, a quantity that conventional inverse solvers do not generally
expose as an explicit, reusable response for active acquisition. Fig.~\ref{fig:Active_grounding}
illustrates this setting and highlights a key acquisition challenge,
where a candidate’s marginal uncertainty alone does not reflect its
value for reconstructing the full target.

Active acquisition has been studied through Bayesian experimental
design \citep{foster_variationalbed_2019,foster_deepadaptivedesign_2021,polyzos_activeradiomap_2024,ivanova_idad_2021}.
Reconstruction-oriented active sensing uses uncertainty- or information-driven
acquisition \citep{zhang_activeMRI_2019,morssy_informed_2024}, and
recent diffusion-based methods derive such criteria from predictive
distributions through variance, diversity, or entropy \citep{barba_diffusionactive_2025,gu_spectrumcartography_2025,nolan_activediffusion_2025}.
Purely marginal criteria prioritize individually uncertain candidates
but do not capture how cross-variable dependencies affect full-target
reconstruction. More structured criteria account for posterior dependence
\citep{iollo_codiff_2025,krause2008}. Notably, AdaSense \citep{elata_adasense_2024}
estimates posterior covariance from empirical samples for linear minimum
mean-squared error (LMMSE)-based criterion, coupling covariance estimation
and acquisition cost to repeated ensemble regeneration as observations
accumulate.

\begin{figure*}[!t]
\centering\includegraphics[scale=0.58]{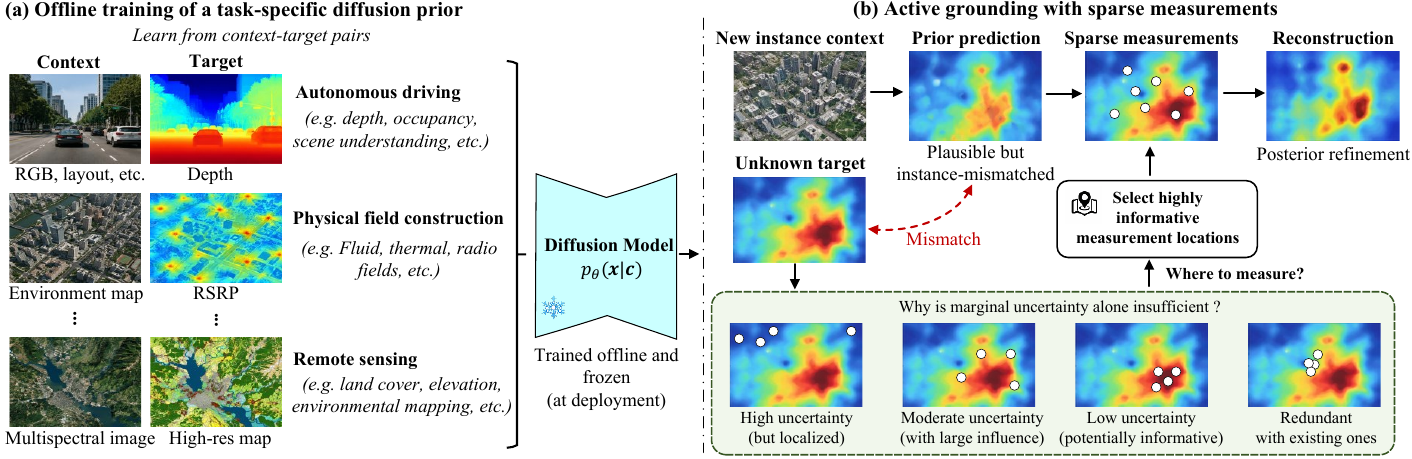}
\caption{Active grounding of a frozen diffusion prior. A task-specific conditional
prior is pretrained offline and frozen. At deployment, context alone
may produce an instance-mismatched prediction, and sparse target measurements
are thus acquired to ground the prior to the current realization.}
\label{fig:Active_grounding}
\end{figure*}

We address these limitations with GPARA, a graph-posterior-aligned
refinement and acquisition framework. Building on graph precision
models \citep{shuman2013gsp,rue2005gmrf,pinti_graphbayes_2025}, GPARA
learns a graph surrogate over prediction residuals, rather than directly
smoothing target values. Unlike AdaSense that infers posterior covariance
from empirical ensembles, GPARA explicitly models residual dependence
through a learned sparse precision, yielding a shared structured posterior
response that both propagates local measurement innovations and evaluates
prospective measurements, with risk weights specifying where uncertainty
reduction is currently valuable. The analytic utility is retained
as an anchor and augmented by a bounded residual calibration.

Our contributions are threefold: (i) we propose a structured graph
residual refinement model that explicitly propagates local measurement
innovations across related variables while keeping the diffusion backbone
frozen; (ii) we develop a response-aligned acquisition that couples
the same graph-induced posterior response with adaptive weighted Bayes
risk, and show that its analytic utility is consistent with expected
one-step reconstruction improvement; and (iii) we combine the analytic
anchor with bounded residual calibration and matrix-free posterior
computation using cached prior samples, stochastic probing, and sparse
linear solves, avoiding dense covariance construction and repeated
posterior-ensemble generation at every acquisition round. 
\section{Background and Problem Formulation}

\subsection{Conditional Diffusion Prior}

We instantiate the formulation with a denoising diffusion probabilistic
model (DDPM) \citep{ho_ddpm_2020,nichol_improved_ddpm_2021}, while
the framework is not restricted to this particular formulation.

Let $\boldsymbol{x}_{0}\in\mathbb{R}^{M}$ denote the clean target
and $\boldsymbol{c}$ the available context. A DDPM defines a $T$-step
forward diffusion process that progressively corrupts $\boldsymbol{x}_{0}$
toward standard Gaussian noise:
\begin{equation}
q(\boldsymbol{x}_{t}\mid\boldsymbol{x}_{0})=\mathcal{N}\left(\boldsymbol{x}_{t};\sqrt{\bar{\alpha}_{t}}\boldsymbol{x}_{0},(1-\bar{\alpha}_{t})\boldsymbol{I}\right),\label{eq:forward_diffusion}
\end{equation}
where $\{\beta_{t}\}^{T}_{t=1}$ is the noise schedule, $\alpha_{t}=1-\beta_{t}$,
$\bar{\alpha}_{t}=\prod^{t}_{s=1}\alpha_{s}$, and $\boldsymbol{I}$
is the identity matrix. A conditional denoising network $\ensuremath{\boldsymbol{\epsilon}_{\theta}(\boldsymbol{x}_{t},t,\boldsymbol{c})}$
is trained to predict the noise injected into $\boldsymbol{x}_{t}$.

At inference time, the reverse process starts from Gaussian noise
$\boldsymbol{x}_{T}\sim\mathcal{N}(\boldsymbol{0},\boldsymbol{I})$.
At each reverse step $t$, the conditional denoiser can first estimate
the clean target as
\begin{equation}
\widehat{\boldsymbol{x}}_{0|t}=\frac{\boldsymbol{x}_{t}-\sqrt{1-\bar{\alpha}_{t}}\boldsymbol{\epsilon}_{\theta}(\boldsymbol{x}_{t},t,\boldsymbol{c})}{\sqrt{\bar{\alpha}_{t}}}.
\end{equation}
The standard reverse transition then uses $\widehat{\boldsymbol{x}}_{0|t}$
and $\boldsymbol{x}_{t}$ to predict $\boldsymbol{x}_{t-1}$:
\begin{equation}
\boldsymbol{x}_{t-1}=a_{t}\widehat{\boldsymbol{x}}_{0|t}+b_{t}\boldsymbol{x}_{t}+\sqrt{\tilde{\beta}_{t}}\boldsymbol{\epsilon},\hspace*{0.3cm}\boldsymbol{\epsilon}\sim\mathcal{N}(\boldsymbol{0},\boldsymbol{I}),\label{eq:reverse transition}
\end{equation}
where $a_{t}$, $b_{t}$, and $\tilde{\beta}_{t}$ are determined
by the predefined diffusion schedule. Once trained, the diffusion
model is frozen and provides the conditional generative prior $p_{\theta}(\boldsymbol{x}_{0}\mid\boldsymbol{c})$
throughout deployment.

\subsection{Problem Formulation}

Although the context $\boldsymbol{c}$ provides rich information cues
about $\boldsymbol{x}_{0}$, environment dynamics, sensing imperfections,
transient states, or fine-scale structures may induce a non-negligible
mismatch between the context-conditioned prior and the actual target.
We retain the pretrained diffusion model $p_{\theta}(\boldsymbol{x}_{0}\mid\boldsymbol{c})$
as an offline prior and use direct measurements to ground it to the
current target.

Since it is costly and often infeasible to obtain dense target measurements,
only a limited budget of $\mathit{\Gamma}\ll M$ measurements can
be acquired. Let $\mathcal{S}$ be the selected observation set with
$|\mathcal{S}|\leq\mathit{\Gamma}$, and let $\boldsymbol{H}_{\mathcal{S}}\in\mathbb{R}^{\mathit{|\mathcal{S}|}\times M}$
be the corresponding observation matrix. The measurement model is
given by
\begin{equation}
\boldsymbol{y}_{\mathcal{S}}=\boldsymbol{H}_{\mathcal{S}}\boldsymbol{x}_{0}+\boldsymbol{\eta}_{\mathcal{S}},\hspace*{0.3cm}\boldsymbol{\eta}_{\mathcal{S}}\sim\mathcal{N}(\boldsymbol{0},\boldsymbol{R}_{\mathcal{S}}),\label{eq:measurement model}
\end{equation}
where $\boldsymbol{R}_{\mathcal{S}}$ is the noise covariance. We
consider equal-cost point observations in this work.

We jointly seek highly informative acquisition and reconstruction
under the sensing budget $\mathit{\Gamma}$. Let $\mathcal{S}_{\mathit{\Gamma}}$
be the observation set induced by an acquisition policy $\pi_{\mathit{\Phi},\psi}$,
which may condition subsequent selections on the information already
obtained. The population-level active grounding objective is
\begin{equation}
(\varPhi^{*},\mathcal{\psi}^{*})\in\arg\min_{\varPhi,\mathcal{\psi}}\hspace{0.1cm}\mathbb{E}\left[\mathcal{L}\left(\boldsymbol{x}_{0},\widehat{\boldsymbol{x}}_{\varPhi}\left(\boldsymbol{c},\boldsymbol{y}_{\mathcal{\mathcal{S}_{\mathit{\Gamma}}}},\mathcal{S}_{\mathit{\Gamma}}\right)\right)\right],\label{eq:channel_vector-1}
\end{equation}
where $\mathit{\Phi}$ and $\psi$ parameterize refinement and acquisition,
respectively, $\widehat{\boldsymbol{x}}_{\varPhi}$ is the target
reconstruction, and $\mathcal{L}(\cdot,\cdot)$ measures reconstruction
error.

The resulting problem thus couples \emph{posterior refinement}, which
determines how newly acquired measurements modify the structured belief
induced by the diffusion prior, with \emph{active acquisition}, which
determines where measurements should be collected to maximize reconstruction
benefit. 
\section{Method}\label{sec:Method}

Fig.~\ref{fig:GPARA} provides an overview of GPARA, which proceeds
in two stages. During active acquisition, a cached prior ensemble
specifies the current spatial risk, and the graph posterior provides
a shared response for conditioning acquired measurements and evaluating
prospective ones through weighted posterior-risk reduction, with a
bounded residual scorer compensating for surrogate mismatch. Newly
obtained measurements are thus incorporated through graph conditioning
without regenerating diffusion samples. After the sensing budget is
reached, final refinement applies the resulting graph-posterior conditioner
to each prediction throughout the reverse diffusion process.

\subsection{Graph-Structured Bayesian Residual Refinement}\label{subsec:Graph-Structured-Bayesian-Refine}

For the clean-state estimate $\widehat{\boldsymbol{x}}_{0|t}$ at
reverse step $t$, we write the remaining discrepancy as
\begin{equation}
\boldsymbol{x}_{0}=\widehat{\boldsymbol{x}}_{0|t}+\boldsymbol{\delta},
\end{equation}
where $\boldsymbol{\delta}\in\mathbb{R}^{M}$ is the residual correction
relative to the prediction. GPARA constructs a shared residual-correction
geometry from a fixed graph-conditioning descriptor $\boldsymbol{c}_{g}$,
including the context $\boldsymbol{c}$ and the known prior $p_{\theta}(\boldsymbol{x}_{0}\mid\boldsymbol{c})$.
A graph network $G_{\phi}$ parameterizes an undirected graph
\begin{equation}
\mathcal{G}=\bigl(\mathcal{V},\mathcal{E},\boldsymbol{C}_{g}\bigr),\hspace*{0.3cm}\boldsymbol{C}_{g}=G_{\phi}(\boldsymbol{c}_{g}),
\end{equation}
where node $i\in\mathcal{\mathcal{V}}$ represents a target variable,
$\mathcal{E}$ captures edges, and $\boldsymbol{C}_{g}=[c_{ij}]\in\mathbb{R}^{M\times M}$
is the symmetric and nonnegative adjacency matrix. A larger $c_{ij}$
encourages measurement corrections to propagate coherently between
coupled variables. Let $\boldsymbol{D}_{g}=\mathrm{diag}(\boldsymbol{C}_{g}\boldsymbol{1})$
and $\ensuremath{\boldsymbol{L}_{g}=\boldsymbol{D}_{g}-\boldsymbol{C}_{g}}$
be the degree matrix and Laplacian, respectively. We define a graph-Gaussian
residual surrogate as
\begin{equation}
q_{\varPhi}(\boldsymbol{\delta}|\boldsymbol{c}_{g})=\mathcal{N}(\boldsymbol{0},\boldsymbol{Q}^{-1}_{g}),\hspace*{0.3cm}\boldsymbol{Q}_{g}=\tau\boldsymbol{I}+\lambda\boldsymbol{L}_{g},\label{eq:graph-Gaussian residual prior}
\end{equation}
where $\tau>0$ controls the residual precision and $\lambda\geq0$
controls the graph coupling. The graph precision is shared across
reverse steps as a reusable correction operator that is learned jointly
across reverse timesteps, rather than parameterized separately for
each noise level.
\begin{prop}
\label{prop:Graph-Bayesian-residual}(Graph-induced residual posterior).
Given the graph residual surrogate (\ref{eq:graph-Gaussian residual prior})
and the observation model (\ref{eq:measurement model}) with $\ensuremath{\boldsymbol{R}_{\mathcal{S}}\succ0}$,
the conditional residual posterior is
\begin{equation}
q_{\varPhi}\left(\boldsymbol{\delta}\mid\widehat{\boldsymbol{x}}_{0|t},\boldsymbol{y}_{\mathcal{S}},\boldsymbol{c}_{g}\right)=\mathcal{N}\left(\boldsymbol{m}_{\mathcal{S},t},\boldsymbol{\Sigma}_{\mathcal{S}}\right),
\end{equation}
where
\begin{equation}
\boldsymbol{m}_{\mathcal{S},t}=\boldsymbol{A}^{-1}_{\mathcal{S}}\boldsymbol{H}^{\mathsf{T}}_{\mathcal{S}}\boldsymbol{R}^{-1}_{\mathcal{S}}\left(\boldsymbol{y}_{\mathcal{S}}-\boldsymbol{H}_{\mathcal{S}}\widehat{\boldsymbol{x}}_{0|t}\right),\hspace*{0.3cm}\boldsymbol{\Sigma}_{\mathcal{S}}=\boldsymbol{A}^{-1}_{\mathcal{S}},\label{eq:mean_n_var}
\end{equation}
and
\begin{equation}
\boldsymbol{A}_{\mathcal{S}}=\boldsymbol{Q}_{g}+\boldsymbol{H}^{\mathsf{T}}_{\mathcal{S}}\boldsymbol{R}^{-1}_{\mathcal{S}}\boldsymbol{H}_{\mathcal{S}}.
\end{equation}
Here, $\mathbf{m}_{\mathcal{S},t}\in\mathbb{R}^{M}$, $\boldsymbol{\Sigma}_{\mathcal{S}}\in\mathbb{R}^{M\times M}$
are the residual posterior mean and covariance conditioned on $\boldsymbol{y}_{\mathcal{S}}$.
Moreover, $\boldsymbol{A}_{\mathcal{S}}\succ0$, and the posterior
is uniquely defined (see Appendix~\ref{subsec:Proof-of-Proposition}
for the proof).
\end{prop}
Proposition~\ref{prop:Graph-Bayesian-residual} establishes an explicit
refinement mechanism that jointly performs observation consistency
and graph propagation within a single update. While measurement innovations
enter at observed locations, $\boldsymbol{A}^{-1}_{\mathcal{S}}$
propagates their effects to graph-coupled unobserved variables. Although
$\boldsymbol{Q}_{g}$ is shared across reverse steps, $\boldsymbol{m}_{\mathcal{S},t}$
remains timestep dependent through the innovation.

The posterior mean also admits a variational interpretation
\begin{equation}
\boldsymbol{m}_{\mathcal{S},t}=\arg\min_{\boldsymbol{\delta}}\;\frac{\tau}{2}\|\boldsymbol{\delta}\|^{2}_{2}+\frac{\lambda}{2}\sum_{(i,j)\in\mathcal{E}}c_{ij}(\delta_{i}-\delta_{j})^{2}+\frac{1}{2}\left\Vert \boldsymbol{H}_{\mathcal{S}}\left(\widehat{\boldsymbol{x}}_{0|t}+\boldsymbol{\delta}\right)-\boldsymbol{y}_{\mathcal{S}}\right\Vert ^{2}_{\mathbf{R}^{-1}_{\mathcal{S}}},\label{eq:variational interpretation}
\end{equation}
where $\|\mathbf{v}\|^{2}_{\mathbf{B}}=\mathbf{v}^{\mathrm{T}}\mathbf{B}\mathbf{v}.$
It balances correction magnitude, graph-structured residual consistency,
and measurement fidelity. The graph-posterior correction is then injected
as
\begin{equation}
\ensuremath{\widetilde{\boldsymbol{x}}_{0|t}=\Pi_{[x_{\min},x_{\max}]}\left(\widehat{\boldsymbol{x}}_{0|t}+\rho_{g}\boldsymbol{m}_{\mathcal{S},t}\right)},\label{eq:correction}
\end{equation}
where $\rho_{g}>0$ controls correction strength, and $\Pi_{[\cdot,\cdot]}$
is element-wise clipping to the normalized range. The refined estimate
$\widetilde{\boldsymbol{x}}_{0|t}$ then replaces the original $\widehat{\boldsymbol{x}}_{0|t}$
in the subsequent reverse transition.

\subsection{Posterior-Aligned Active Acquisition}

\begin{figure*}[!t]
\centering\includegraphics[scale=0.507]{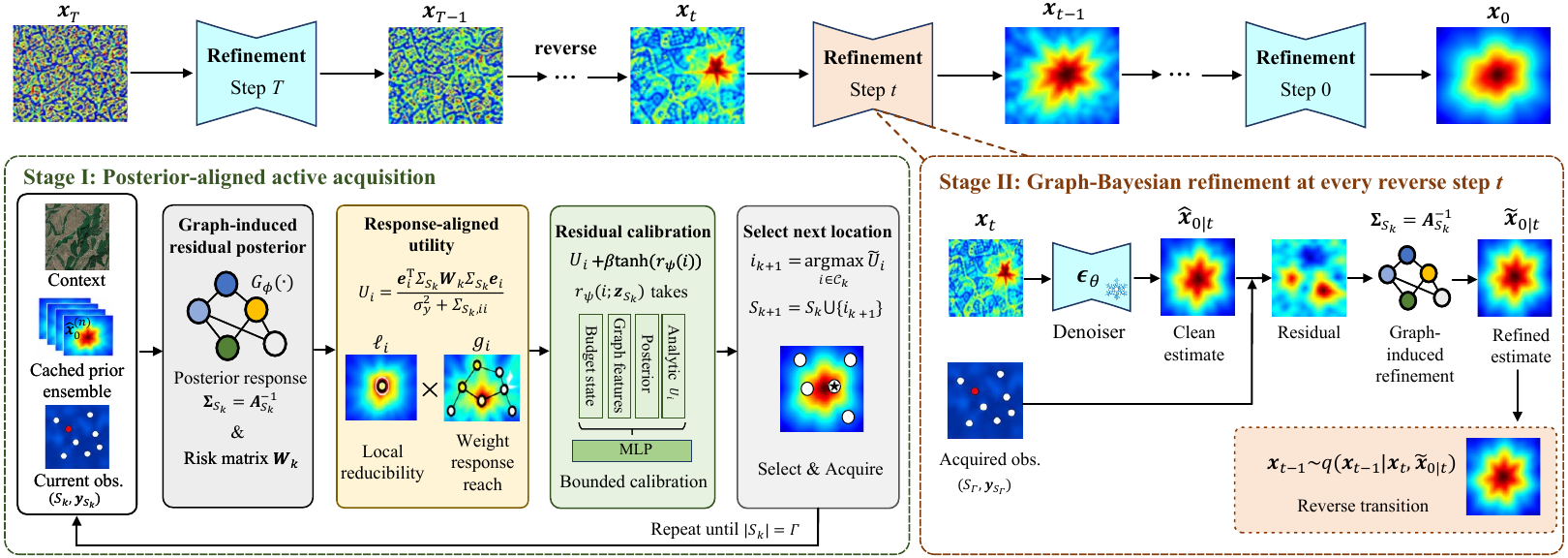} \caption{Overview of GPARA. Stage~I performs response-aligned active acquisition:
the graph posterior specifies how information propagates, while a
state-adaptive risk matrix specifies where uncertainty reduction is
valuable. A bounded residual scorer calibrates the analytic utility.
After reaching the sensing budget, Stage~II fixes the acquired observations
and reuses the same graph-posterior conditioner to refine prediction
throughout the reverse diffusion process.}
\label{fig:GPARA}
\end{figure*}

Let $\mathcal{S}_{k}$ be the observation set of acquisition round
$k$. GPARA evaluates prospective measurements through the same graph-posterior
covariance. Let $\boldsymbol{W}_{k}=\mathrm{diag}(w_{k,1},...,w_{k,M})\succeq0$
denote a diagonal risk-weighting matrix to specify the risk emphasis
of target variables, as detailed in Appendix~\ref{subsec:risk-weighting matrix}.
We define the weighted posterior risk as
\begin{equation}
\mathcal{R}(\mathcal{S}_{k})=\mathrm{tr}\left(\boldsymbol{W}_{k}\boldsymbol{\Sigma}_{\mathcal{S}_{k}}\right).
\end{equation}
While $\boldsymbol{\Sigma}_{\mathcal{S}_{k}}$ characterizes how a
prospective measurement propagates across the target through cross-variable
posterior dependence, $\boldsymbol{W}_{k}$ specifies where such reduction
is currently most valuable.

Consider an unobserved candidate $i\notin\mathcal{S}_{k}$ with measurement
noise variance $\sigma^{2}_{y}$, and let $\boldsymbol{e}_{i}\in\mathbb{R}^{M}$
be the $i$th canonical basis vector. Holding the current graph and
risk weights, its one-step utility is
\begin{equation}
U_{i}(\mathcal{S}_{k})=\mathcal{R}(\mathcal{S}_{k})-\mathcal{R}(\mathcal{S}_{k}\cup\{i\})=\frac{\boldsymbol{e}^{\mathsf{T}}_{i}\boldsymbol{\Sigma}_{\mathcal{S}_{k}}\boldsymbol{W}_{k}\boldsymbol{\Sigma}_{\mathcal{S}_{k}}\boldsymbol{e}_{i}}{\sigma^{2}_{y}+\Sigma_{\mathcal{S}_{k},ii}}=\underbrace{\frac{\Sigma^{2}_{\mathcal{S}_{k},ii}}{\sigma^{2}_{y}+\Sigma_{\mathcal{S}_{k},ii}}}_{\text{local reducibility \ensuremath{\ensuremath{\ell_{i}}}}}\times\underbrace{\frac{\sum^{M}_{j=1}w_{k,j}\Sigma^{2}_{\mathcal{S}_{k},ji}}{\Sigma^{2}_{\mathcal{S}_{k},ii}}}_{\text{weight response reach \ensuremath{g_{i}}}},
\end{equation}
where $\Sigma_{\mathcal{S},ji}$ denotes the posterior covariance
between node $i$ and $j$. The local factor $\ell_{i}$ measures
uncertainty removable at the candidate itself; $g_{i}$ captures how
broadly that information propagates through their weight correlations.
A moderately uncertain location can be more valuable than a highly
uncertain one if it influences a broader unresolved region. When $\boldsymbol{W}_{k}=\boldsymbol{I}$
and $\sigma^{2}_{y}=0$, it reduces to the LMMSE criterion used in
AdaSense \citep{elata_adasense_2024}. We next characterize how the
shared graph-posterior response relates posterior refinement to prospective
measurement value.
\begin{prop}
\label{prop:Proposition 2}(Posterior Propagation-Value Identity).
For an arbitrary reverse step $t$, define the posterior response
vector $\boldsymbol{u}_{i}=\boldsymbol{\Sigma}_{\mathcal{S}_{k}}\boldsymbol{e}_{i}/(\sigma^{2}_{y}+\Sigma_{\mathcal{S}_{k},ii})$
and the prospective innovation $\ensuremath{\nu_{i}=\boldsymbol{y}_{i}-\boldsymbol{e}^{\mathsf{T}}_{i}\left(\widehat{\boldsymbol{x}}_{0|t}+\boldsymbol{m}_{\mathcal{S}_{k},t}\right).}$
The posterior-mean change induced by measuring $i$ is thus $\Delta\widehat{\boldsymbol{x}}_{i}=\boldsymbol{u}_{i}\nu_{i}$,
and
\[
U_{i}(\mathcal{S}_{k})=(\sigma^{2}_{y}+\Sigma_{\mathcal{S}_{k},ii})\left\Vert \boldsymbol{u}_{i}\right\Vert ^{2}_{\boldsymbol{W}_{k}}=\mathbb{E}_{q_{\varPhi}}\left[\left\Vert \Delta\widehat{\boldsymbol{x}}_{i}\right\Vert ^{2}_{\boldsymbol{W}_{k}}\mid\mathcal{F}_{k}\right],
\]
where $\mathcal{F}_{k}$ denotes the acquisition state, including
all known quantities such as $\boldsymbol{c}_{g}$ (see Appendix~\ref{subsec:Proof-of-Proposition-1}).
\end{prop}
Proposition~\ref{prop:Proposition 2} shows acquisition-refinement
alignment: the same posterior response that propagates a newly acquired
innovation also determines its expected weighted acquisition value.
\begin{cor}
\label{cor:Ranking Consistency}(Matched-Model Ranking Consistency).
For the unclipped graph-refined estimate $\widetilde{\boldsymbol{x}}_{0|t,\mathcal{S}_{k}}$,
define the one-step reconstruction improvement induced by acquiring
candidate $i$ as
\[
\Delta_{i}(\rho_{g})=\left\Vert \boldsymbol{x}_{0}-\widetilde{\boldsymbol{x}}_{0|t,\mathcal{S}_{k}}\right\Vert ^{2}_{\boldsymbol{W}_{k}}-\left\Vert \boldsymbol{x}_{0}-\widetilde{\boldsymbol{x}}_{0|t,\mathcal{S}_{k}\cup\{i\}}\right\Vert ^{2}_{\boldsymbol{W}_{k}}.
\]
Under the matched graph-Gaussian surrogate,
\[
\mathbb{E}_{q_{\varPhi}}\left[\Delta_{i}(\rho_{g})\mid\ensuremath{\mathcal{F}_{k}}\right]=(2\rho_{g}-\rho^{2}_{g})U_{i}(\mathcal{S}_{k}).
\]
For $0<\rho_{g}<2$, maximizing $U_{i}(\mathcal{S}_{k})$ thus maximizes
its expected improvement (see Appendix~\ref{subsec:Proof-of-Corollary}).
\end{cor}
Corollary~\ref{cor:Ranking Consistency} further shows that, under
the matched surrogate, the candidate with the largest utility $U_{i}(\mathcal{S}_{k})$
is expected to maximize the one-step improvement in the same loss.
This result motivates using such utility as the acquisition anchor.
In the diffusion process, where clipping, finite-step denoising, approximate
linear solves, and surrogate mismatch may break the matched-model
assumptions, we use $\ensuremath{U_{i}}$ as an explicit acquisition
anchor and permit residual calibration from it.

Accordingly, GPARA retains $U_{i}(\mathcal{S}_{k})$ as the analytic
anchor and defines the final acquisition score as
\[
\widetilde{U}_{i}(\mathcal{S}_{k})=\bar{U}_{i}(\mathcal{S}_{k})+\beta_{\mathrm{acq}}\tanh\left(r_{\psi}(i;\ensuremath{\mathbf{z}_{\mathcal{S}_{k}}})\right),
\]
where $r_{\psi}(\cdot)$ is the scorer network, $\bar{U}_{i}$ is
the normalized analytic utility, $\mathbf{z}_{\mathcal{S}_{k}}$ denotes
the acquisition state, including context, graph-posterior statistics,
observation information, and the budget, and $\beta_{\mathrm{acq}}>0$
bounds its magnitude. When $r_{\psi}(i;\ensuremath{\mathbf{z}_{\mathcal{S}_{k}}})=0$,
it reduces exactly to the analytic anchor.

Starting from $\mathcal{S}_{0}=\varnothing$, GPARA progressively
acquires measurements according to the current calibrated utility.
The underlying single-point selection rule is
\[
i_{k+1}=\arg\max_{i\in\mathcal{C}_{k}}\widetilde{U}_{i}(\mathcal{S}_{k}),\hspace*{0.3cm}\mathcal{S}_{k+1}=\mathcal{S}_{k}\cup\{i_{k+1}\},
\]
where $\mathcal{C}_{k}\subseteq\mathcal{V}\setminus\mathcal{S}_{k}$
denotes the admissible unobserved candidates at step $k$. For computational
efficiency, we additionally provide a batched acquisition strategy
that selects a small group of candidates at each round $k$ and jointly
assimilates them before refreshing the posterior and acquisition scores.
The single-point utility remains the primary acquisition criterion;
batching is introduced only as a computational tradeoff to reduce
update frequency. Further details are provided in Appendix~\ref{subsec:Batched-Acquisition}.

\paragraph{Implementation and training.}

GPARA consists of two networks trained offline. The graph predictor
$G_{\phi}$ learns edge conductances through differentiable unrolled
graph conditioning with reconstruction and graph-regularization objectives.
The residual scorer $r_{\psi}$ is trained to calibrate the analytic
utility using a training-only one-step improvement teacher and distillation,
ranking, and calibration losses. No parameters are updated at deployment.
A small prior ensemble is generated once and cached for adaptive risk
weighting. Posterior corrections use fixed-depth Jacobi iterations,
and covariance-dependent utilities are evaluated matrix-free using
stochastic probing and PCG without forming a dense posterior covariance
matrix. Algorithm~\ref{alg:gpara} summarizes the full procedure,
with numerical settings and network training deferred to Appendix~\ref{sec:Proofs-and-Numerical} and \ref{sec:Network-Architectures--Offline}.

\begin{algorithm}[t]
\caption{GPARA: Graph-Posterior-Aligned Active Acquisition and Reconstruction}
\label{alg:gpara}
\small
\begin{algorithmic}[1]

\Require Context $\boldsymbol{c}$; frozen networks $\boldsymbol{\epsilon}_{\theta}$, $G_{\phi}$, and $r_{\psi}$; sensing budget $\mathit{\Gamma}$; ensemble size $L$; number of diffusion samples $N$; measurement interface and noise model

\Ensure Reconstruction of final target $\hat{\boldsymbol{x}}_0$ and the acquired measurement set $\mathcal{S}_\mathit{\Gamma}$

\Statex \textbf{Initialization}

\State Set
$k \gets 0$, and $\mathcal{S}_0 \gets \varnothing$. Generate $\mathbf{C}_g$ from $G_{\phi}$,
compute $\mathbf{L}_g$ and $\mathbf{Q}_g \gets \tau\mathbf{I}+\lambda\mathbf{L}_g$

\State Generate and cache $L$ DDPM prior samples $\{\widehat{\boldsymbol{x}}^{(l)}_{0}\}^{L}_{l=1}$ without measurements

\Statex \textbf{Stage I: Posterior-aligned active acquisition}

\While{$|\mathcal{S}_k|<\mathit{\Gamma}$}

    \State Construct
    $\boldsymbol{A}_{\mathcal{S}_k}
    \gets
    \mathbf{Q}_g
    +
    \mathbf{H}_{\mathcal{S}_k}^{\mathsf T}
    \mathbf{R}_{\mathcal{S}_k}^{-1}
    \mathbf{H}_{\mathcal{S}_k}$, and obtain
    $\{\widetilde{\boldsymbol{x}}^{(l)}_{0|\mathcal{S}_{k}}\}^{L}_{l=1}$

    \State Compute the ensemble mean and empirical variance, and
    construct $\boldsymbol{W}_k$ (see Appendix~\ref{subsec:risk-weighting matrix})

    \State Estimate $U_i(\mathcal{S}_k)$ and covariance statistics using stochastic probing and PCG (see Appendix~\ref{subsec:Matrix-Free-Graph-Solves})

    \State Compute, for each admissible candidate,
    $\widetilde{U}_i(\mathcal{S}_k)
    \gets
    \bar{U}_i
    +
    \beta_{\mathrm{acq}}\tanh\!\bigl(
    r_{\psi}(i;\mathbf{z}_{\mathcal{S}_k})
    \bigr)$

    \State Select $i_{k+1}\gets \arg\max_{i\in\mathcal{C}_{k}}    \widetilde{U}_i(\mathcal{S}_k)$, acquire $y_{i_{k+1}}$ and $\mathcal{S}_{k+1}\gets\mathcal{S}_k\cup\{i_{k+1}\}$

    \State $k\gets k+1$

\EndWhile

\Statex \textbf{Stage II: Final graph-refined diffusion reconstruction}

\State Reconstruct $\boldsymbol{A}_{\mathcal{S}_\mathit{\Gamma}}$ using the final
observation set and keep it fixed throughout reconstruction

\State Initialize $N$ reverse diffusion trajectories from Gaussian noise

\For{each reverse transition $t$}

    \For{$n=1,\ldots,N$}

        \State Obtain the clean-state prediction
        $\widehat{\boldsymbol{x}}_{0|t}^{(n)}$ from
        $\boldsymbol{\epsilon}_{\theta}
        (\hat{\boldsymbol{x}}_t^{(n)},t,\mathbf{c})$

        \State Solve the graph-conditioning system
        for $\widehat{\mathbf{m}}_{\mathcal{S}_\mathit{\Gamma},t}^{(n)}$
        with
        $\mathbf{y}_{\mathcal{S}_\mathit{\Gamma}}$ by Proposition~\ref{prop:Graph-Bayesian-residual}

        \State Update
        $\widetilde{\boldsymbol{x}}_{0|t,\mathcal{S}_\mathit{\Gamma}}^{(n)}
        \gets
        \Pi_{[x_{\min},x_{\max}]}
        \!\left(
        \widehat{\boldsymbol{x}}_{0|t}^{(n)}
        +
        \rho_g\widehat{\mathbf{m}}_{\mathcal{S}_\mathit{\Gamma},t}^{(n)}
        \right)$

        \State Apply the DDPM transition
        with $\widetilde{\boldsymbol{x}}_{0|t,\mathcal{S}_\mathit{\Gamma}}^{(n)}$ to obtain
        $\hat{\boldsymbol{x}}_{t-1}^{(n)}$ by Eq.~(\ref{eq:reverse transition})

    \EndFor

\EndFor

\State Collect the terminal samples
$\{\hat{\boldsymbol{x}}_0^{(n)}\}_{n=1}^{N}$ and compute
$\hat{\boldsymbol{x}}_0
\gets
N^{-1}\sum_{n=1}^{N}\hat{\boldsymbol{x}}_0^{(n)}$

\State \Return
$\hat{\boldsymbol{x}}_0,\mathcal{S}_\mathit{\Gamma}$

\end{algorithmic}
\end{algorithm}
\section{Experiments}\label{sec:experiments}

We evaluate \ProposedM on two representative tasks: radio field reconstruction
(RFR) and RGB-guided depth completion (DC). Each task uses a separately
pretrained DDPM as the frozen backbone, with a five-step reverse process
for generation. All experiments are conducted on an NVIDIA RTX 4060
Ti GPU (16 GB) with an Intel Core i7-10700K CPU and 32 GB of RAM.

\subsection{Experimental Setup}\label{sec:experimental_setup}

\paragraph{Datasets and metrics.}

RFR uses RadioMapSeer IRT4 \citep{RadioMapSeer_yapar2024}, with building
geometry and transmitter configuration as context. Each target is
a $128\times128$ radio map, with $1,122/140/140$ samples for training/validation/test.
DC uses NYU Depth V2 \citep{Silberman:ECCV12}, with RGB images as
context. The RGB-depth pairs are resized to $192\times256$, yielding
$1,349$ training images and $100$ test images from $35$ unseen
scenes. We report root mean squared error (RMSE) and mean absolute
error (MAE) for both tasks, with peak signal-to-noise ratio (PSNR)
and structural similarity index (SSIM) for RFR, and absolute relative
error (AbsRel) and threshold accuracy $\delta_{1}$ for DC.

\paragraph{Baselines.}

Fixed-mask methods share the same pretrained DDPM within each task
and receive identical random measurements. Baselines are DPS \citep{chung_dps_2023},
DDNM \citep{wang_ddnm_2023}, FPS \citep{dou_fps_2024}, and DiffPIR
\citep{zhu_diffpir_2023}. For active acquisition, we compare with
ADS \citep{nolan_activediffusion_2025} and AdaSense \citep{elata_adasense_2024}.
AdaSense is paired with each posterior backend, including the proposed
graph-Bayesian posterior.

\begin{figure*}[!t]
\centering \includegraphics[scale=0.4]{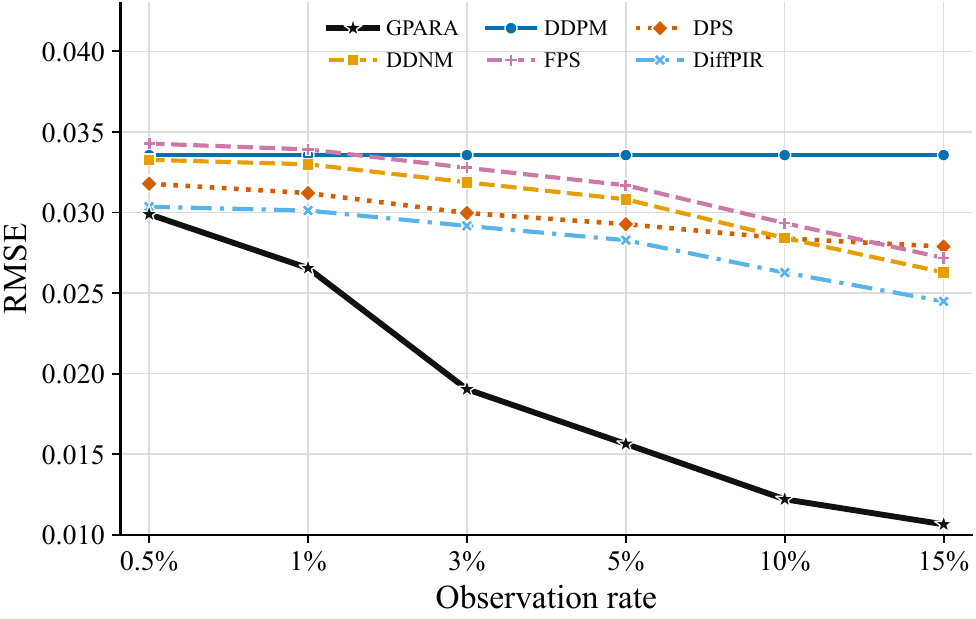}
\includegraphics[scale=0.4]{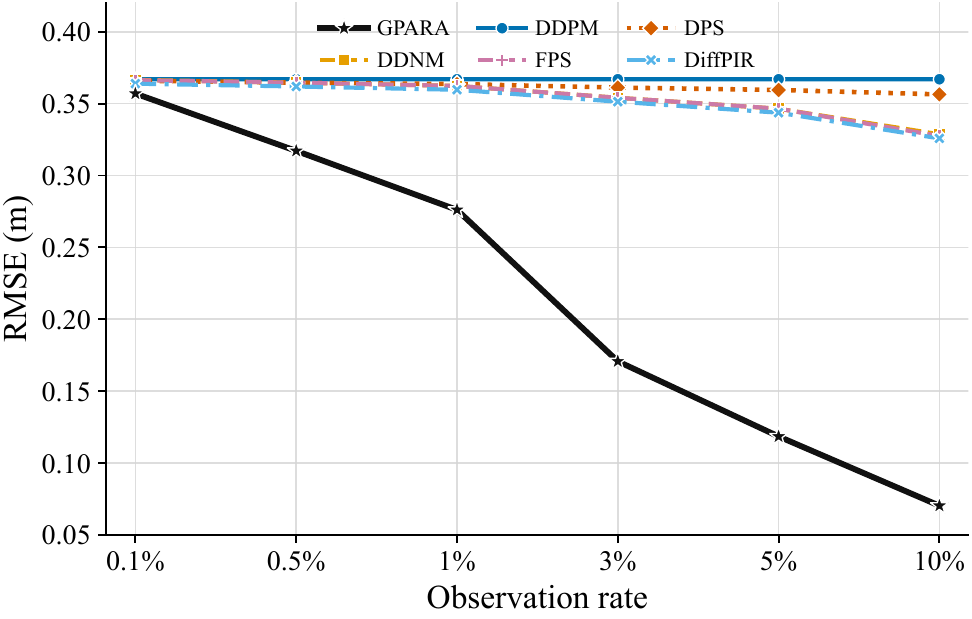}\caption{Posterior refinement across observation rates. RMSE as a function
of the fraction of observed target locations for (a) RFR and (b) DC.
All methods use the same frozen DDPM prior and identical randomly
sampled observations at every observation rate.}
\label{fig:posterior_rates}
\end{figure*}

\begin{table}
\caption{Posterior refinement under 491 matched random observations, approximately
$3\%$ of RFR and $1\%$ of DC target entries. RMSE is reported as
mean $\pm$ standard deviation across five independent seeds; other
metrics show the mean values.}

\centering\label{tab:posterior_refinement} \resizebox{1\textwidth}{!}{%
\renewcommand\arraystretch{1.05}
\begin{tabular}{
p{2.2cm}<{\centering}|
p{2.5cm}<{\centering}p{1.35cm}<{\centering}p{1.35cm}<{\centering}p{1.35cm}<{\centering}|
p{2.5cm}<{\centering}p{1.35cm}<{\centering}p{1.35cm}<{\centering}p{1.35cm}<{\centering}
}
\hline

\multirow{2}{*}{Method} &
\multicolumn{4}{c|}{Radio-field reconstruction} &
\multicolumn{4}{c}{Depth completion} \\

\cline{2-9} 

& RMSE $\downarrow$ & MAE $\downarrow$  & SSIM $\uparrow$ & PSNR $\uparrow$
& RMSE $\downarrow$ & MAE $\downarrow$ & AbsRel $\downarrow$ & $\delta_1$ $\uparrow$ \\

\hline

DDPM
& 0.035 $\pm$ 0.003 & 0.025 & 0.926 & 30.2 & 0.350 $\pm$ 0.005 & 0.273 & 0.113 & 0.878 \\

DPS
& 0.031 $\pm$ 0.002 & 0.021 & 0.937 & 31.1 & 0.347 $\pm$ 0.005 & 0.271 & 0.112 & 0.881\\

DDNM
& 0.033 $\pm$ 0.003 &0.023 &0.919 &30.6 &0.344 $\pm$ 0.005 & 0.267& 0.110& 0.883 \\

FPS
& 0.034 $\pm$ 0.003 & 0.024 & 0.921 & 30.4 & 0.345 $\pm$ 0.004 &0.269 &0.111 & 0.883 \\

DiffPIR
& 0.030 $\pm$ 0.002 & 0.021 & 0.934 & 31.2 & 0.342 $\pm$ 0.004 & 0.266 & 0.110 & 0.885 \\

GPARA (Ours)
& \textbf{0.019 $\pm$ 0.0007} & \textbf{0.011} & \textbf{0.955} & \textbf{34.8} & \textbf{0.262 $\pm$ 0.002} & \textbf{0.196} & \textbf{0.081} & \textbf{0.951} \\

\hline
\end{tabular}
}
\end{table}

\subsection{Fixed-Observation Posterior Refinement}\label{sec:posterior_refinement_experiment}

We first evaluate posterior refinement under identical randomly sampled
observations.

Table~\ref{tab:posterior_refinement} compares how each method uses
the available evidence. On RFR, \ProposedM reduces RMSE and MAE by
$36.7\%$ and $47.6\%$, respectively, relative to DiffPIR, the strongest
baseline here. On DC, \ProposedM reduces RMSE and AbsRel by $23.4\%$
and $26.4\%$, respectively, while increasing $\delta_{1}$ by $6.6$
percentage points. Relative to DDPM prior, the RMSE reductions are
$45.7\%$ on RFR and $25.1\%$ on DC. The consistent improvements
indicate that the proposed graph-posterior refinement makes more effective
use of the measurements. Note that all solvers are evaluated in an
accelerated few-step regime, with additional results under larger
reverse-step budgets reported in Appendix~\ref{sec:Additional-results}.

Fig.~\ref{fig:posterior_rates} further evaluates RMSE across observation
rates. It is observed that \ProposedM benefits more strongly from
additional measurements than the competing baselines as the observation
rate increases. On RFR, the gain over the strongest baseline grows
from a modest margin in sparse regime to over $60\%$ at higher rates;
on DC, the gain similarly expands and reaches nearly $80\%$.

\subsection{Posterior-Aligned Active Acquisition}\label{sec:active_acquisition_experiment}

\begin{figure}[!t]
\centering\includegraphics[scale=0.358]{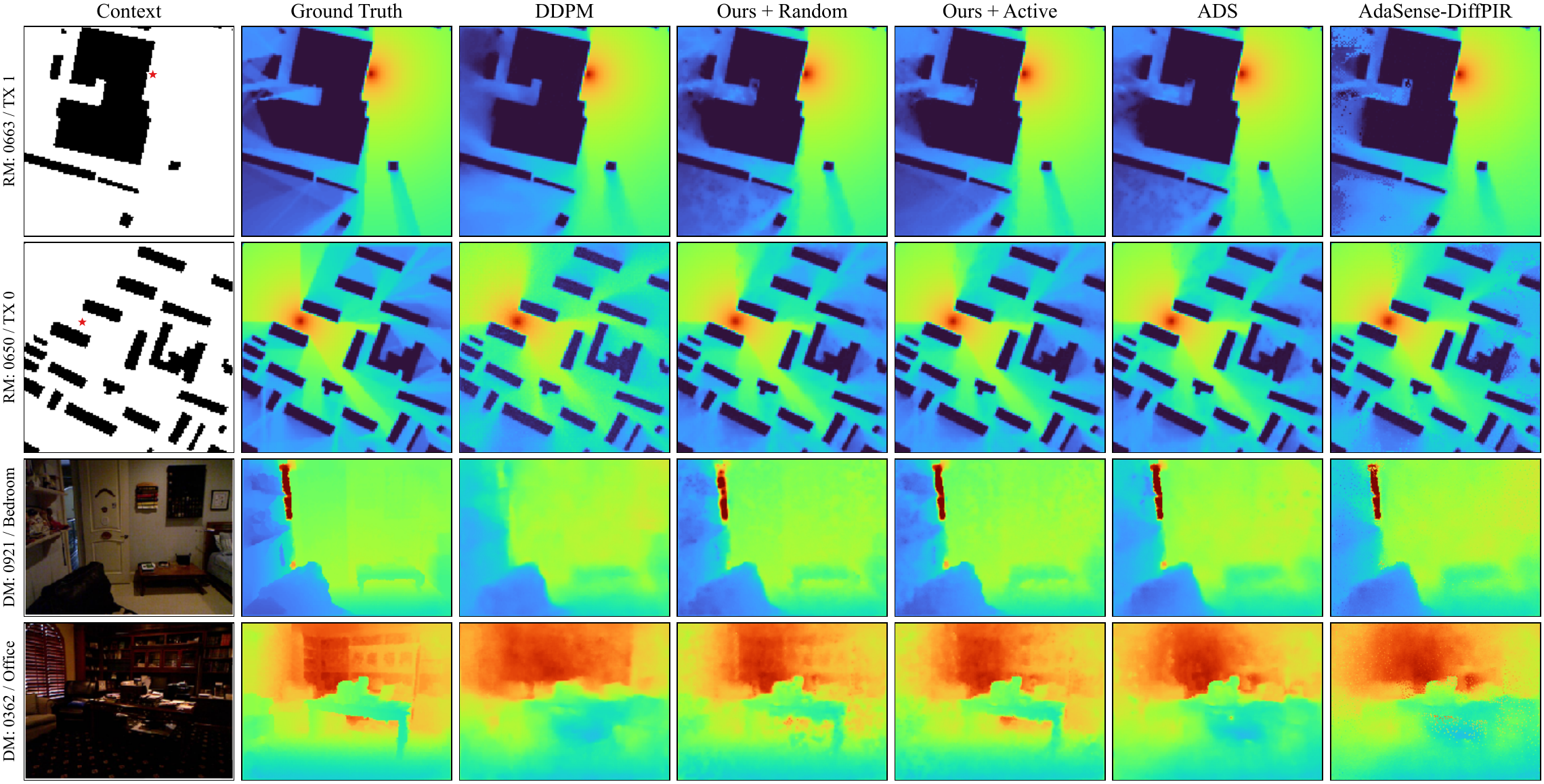}
\caption{Qualitative comparison of active acquisition and reconstruction. The
first two rows show RFR examples with $10\%$ observations, and the
last two rows are DC examples under $5\%$. }
\label{fig:active_qualitative}
\end{figure}

\begin{table}
\caption{Active acquisition with a matched 491-point sensing budget across
five independent seeds.}

\centering\label{tab:active_acquisition} \resizebox{1\textwidth}{!}{%
\renewcommand\arraystretch{1.05}
\begin{tabular}{
p{2.9cm}<{\centering}|
p{2.5cm}<{\centering}p{1.35cm}<{\centering}p{1.35cm}<{\centering}p{1.35cm}<{\centering}|
p{2.5cm}<{\centering}p{1.35cm}<{\centering}p{1.35cm}<{\centering}p{1.35cm}<{\centering}
}
\hline

\multirow{2}{*}{Method} &
\multicolumn{4}{c|}{Radio-field reconstruction} &
\multicolumn{4}{c}{Depth completion} \\

\cline{2-9}

& RMSE $\downarrow$ & MAE $\downarrow$ & SSIM $\uparrow$ & PSNR $\uparrow$
& RMSE $\downarrow$ & MAE $\downarrow$ & AbsRel $\downarrow$ & $\delta_1$ $\uparrow$ \\

\hline

ADS
& 0.026 $\pm$ 0.001 & 0.018 & 0.947 & 33.1
& 0.338 $\pm$ 0.002 & 0.267 & 0.111 & 0.882 \\

AdaSense + DPS
& 0.026 $\pm$ 0.001 & 0.018 & 0.947 & 33.2
& 0.339 $\pm$ 0.002 & 0.268 & 0.111 & 0.881 \\

AdaSense + DDNM
& 0.028 $\pm$ 0.002 & 0.020 & 0.937 & 32.7
& 0.336 $\pm$ 0.002 & 0.267 & 0.110 & 0.883 \\

AdaSense + FPS
& 0.029 $\pm$ 0.002 & 0.021 & 0.935 & 32.5
& 0.336 $\pm$ 0.002 & 0.266 & 0.110 & 0.883 \\

AdaSense + DiffPIR
& 0.026 $\pm$ 0.001 & 0.018 & 0.948 & 33.4
& 0.334 $\pm$ 0.002 & 0.264 & 0.110 & 0.885 \\

AdaSense + Graph
& 0.018 $\pm$ 0.0006 & 0.013 & 0.956 & 36.4
& 0.268 $\pm$ 0.007 & 0.211 & 0.089 & 0.933 \\

GPARA (Ours)
& \textbf{0.014 $\pm$ 0.0004} & \textbf{0.009} & \textbf{0.962} & \textbf{38.1}
& \textbf{0.236 $\pm$ 0.006} & \textbf{0.187} & \textbf{0.079} & \textbf{0.958} \\

\hline
\end{tabular}
}
\end{table}

We next allow each method to actively select measurements under the
same sensing budget. “AdaSense + Graph” uses the proposed graph-based
reconstruction backend. All methods are evaluated on the common prediction
region shared across policies.

From Table~\ref{tab:active_acquisition}, relative to the original
DDPM prior, \ProposedM reduces RMSE by $60.0\%$ on RFR and $32.6\%$
on DC. Compared with random sensing, response-aligned acquisition
provides a further $26.3\%$ and $9.9\%$ RMSE reduction, respectively.
\ProposedM also outperforms the strongest baseline by $46.2\%$ and
$29.3\%$ on RMSE. Importantly, compared with AdaSense+Graph under
the same backend, \ProposedM still achieves $22.2\%$ and $11.9\%$
lower RMSE, respectively, supporting the benefit of the proposed response-aligned
acquisition strategy beyond the reconstruction backend itself.

The qualitative examples in Fig.~\ref{fig:active_qualitative} indicate
that \ProposedM more faithfully recovers spatial variations and structural
boundaries, with fewer reconstruction artifacts. The improvement of
Ours + Active over Ours + Random further illustrates that the gain
arises not only from posterior refinement, but also from acquiring
more informative measurements. These examples should be interpreted
as visual support for the quantitative trends rather than as independent
evidence of general superiority.

Fig.~\ref{fig:active_conditions} evaluates measurement efficiency
across observation rates on RFR. GPARA achieves the lowest MAE at
every observation rate. The gap over the conventional active baselines
widens rapidly beyond the sparsest regime. AdaSense + Graph also improves
substantially as the budget increases, confirming the benefit of the
shared graph reconstruction backend.

Table~\ref{tab:dc_efficiency} evaluates acquisition-stage computational
cost on DC. Implementation configuration of GPARA is given in Appendix~\ref{subsec:Matrix-Free-Graph-Solves}.
It requires only 20 sample-equivalent NFEs, substantially reducing
dependence on repeated denoiser evaluations, but posterior computations
remain a non-negligible part of its runtime. While slower than ADS
in wall-clock time, it uses $54.6\%$ less peak memory, remains comparable
to the AdaSense variants, and is $1.71\times$ faster than AdaSense
+ Graph.

\begin{figure}[t]
\centering

\begin{minipage}[t]{0.49\textwidth}
    \vspace{0pt}
    \centering
    \includegraphics[width=\linewidth]{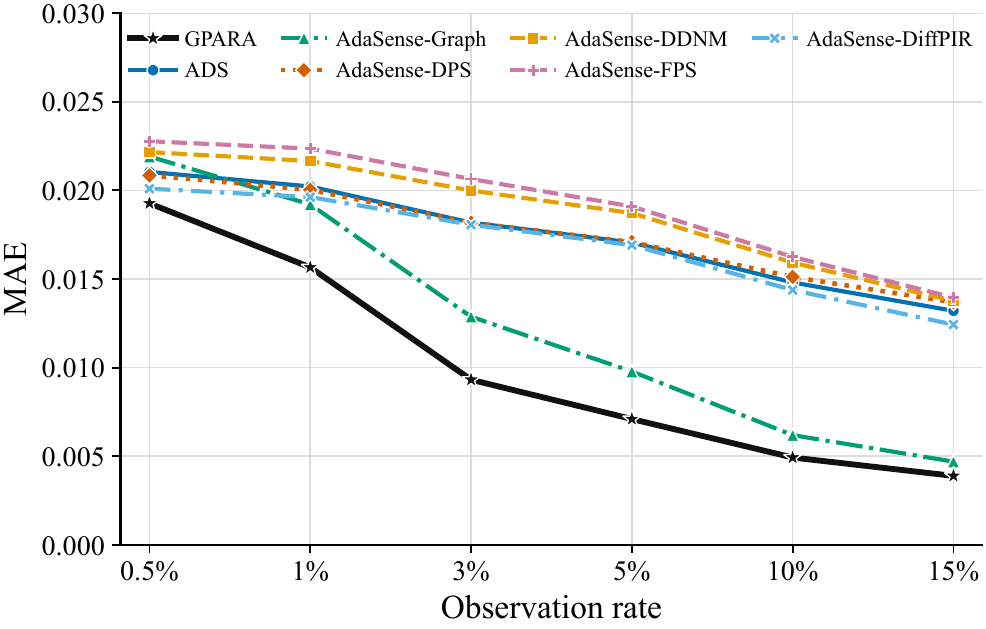}
    \vspace{-4mm}
    \captionof{figure}{
        Active acquisition performance over the common prediction region across varying observation rates on RFR.
    }
    \label{fig:active_conditions}
\end{minipage}
\hfill
\begin{minipage}[t]{0.49\textwidth}
\small
    \vspace{0pt}
    \centering
    \captionof{table}{
        Acquisition efficiency on DC. Time denotes acquisition-stage wall-clock time per image (mean \(\pm\) standard deviation); NFE is the sample-equivalent number of denoiser evaluations, and Mem. is peak GPU memory usage.
    }
    \label{tab:dc_efficiency}
	\resizebox{1\textwidth}{!}{%
    \renewcommand{\arraystretch}{1.05}
    \begin{tabular}{p{2cm}<{\raggedright}p{1.7cm}<{\centering}p{0.5cm}<{\centering}p{0.9cm}<{\centering}}
	\hline
        \multirow{2}{*}{Method}
        & \multirow{2}{*}{\makecell[c]{Time \\(s/img)}}
        & \multirow{2}{*}{NFE}
        & \multirow{2}{*}{\makecell[c]{Mem. \\(MiB)}} \\
		& & & \\
        \hline
        ADS
        & $\mathbf{3.14 \pm 0.24}$
        & 252
        & 1493.6
        \\
        AdaS.+DPS
        & 14.7 $\pm$ 0.21
        & 1240
        & 1497.4
        \\
        AdaS.+DDNM
        & 6.72 $\pm$ 0.15
        & 1240
        & 664.4
        \\
        AdaS.+FPS
        & 6.87 $\pm$ 0.15
        & 1240
        & 668.9
        \\
        AdaS.+DiffPIR
        & 6.88 $\pm$ 0.15
        & 1240
        & 663.6
        \\
        AdaS.+Graph
        & 11.5 $\pm$ 0.27
        & 1240
        & 671.7
        \\
        GPARA
        & \textbf{6.73 $\pm$ 0.18}
        & \textbf{20}
        & 677.9
        \\
        \hline
    \end{tabular}
}
\end{minipage}

\end{figure}

\subsection{Mechanism and Ablation Studies}\label{sec:mechanism_ablation}

\paragraph{Posterior refinement mechanism.}

Table~\ref{tab:mechanism_ablation1}(a) separates the effects of
graph coupling and correction placement under matched random observations.
With step-wise injection, a uniform graph reduces RMSE by $40.4\%$
on RFR and $22.8\%$ on DC relative to pointwise conditioning. Learning
the edge conductances provides a further $2.1\%$ and $1.9\%$ reduction
under the same injection schedule. Moving from final-only to step-wise
correction improves the learned-graph variant by $18.0\%$ and $14.7\%$,
respectively. These comparisons show that substantial improvements
can be obtained through graph coupling and repeated measurement correction,
while learned conductances yield smaller but consistent additional
improvement gains.

\paragraph{Acquisition-score components.}

Table~\ref{tab:mechanism_ablation1}(b) evaluates the acquisition-score
components with a shared reconstruction backend. Replacing local reducibility
with the unweighted utility reduces RMSE by $33.7\%$ on RFR and $1.9\%$
on DC. Risk weighting provides a further $3.7\%$ and $2.7\%$ reduction,
while adding residual calibration improves the weighted analytic policy
by $9.6\%$ and $7.8\%$, respectively. To assess the value of the
explicit analytic anchor, we additionally train a direct learn scorer
$U^{ln}_{i}$ that predicts the candidate score without an additive
analytic term, using the same loss functions as the calibrated scorer.
Relative to this baseline, the weighted analytic policy reduces RMSE
by $9.8\%$ on RFR and $4.5\%$ on DC, whereas the calibrated policy
achieves reductions of $18.4\%$ and $11.9\%$. These results support
retaining the analytic posterior-response anchor and learning residual
corrections, rather than learning the entire acquisition score directly.

\begin{table}
\caption{Mechanism ablations of \ProposedM. (a) Compares posterior mechanisms
under identical random observations. (b) Compares acquisition criteria
while keeping others fixed. }

\centering\label{tab:mechanism_ablation1} \small
\setlength{\tabcolsep}{3pt}
\renewcommand{\arraystretch}{1.15}

\begin{minipage}[t]{0.48\linewidth}
\centering
\textbf{(a) Posterior refinement mechanism}

\vspace{2pt}

\begin{tabular}{
p{2cm}<{\raggedright}
p{1.8cm}<{\centering}
p{1.1cm}<{\centering}
p{1.1cm}<{\centering}
}
\hline
Model & Injection & RFR & DC \\

\hline

Point Bayes  & Step-wise
& 0.0327 & 0.346
\\

Uniform graph  & Final-only
& 0.0240 & 0.309
\\

Learn Graph & Final-only
& 0.0233 & 0.307
\\

Uniform graph  & Step-wise
& 0.0195 & 0.267
\\

Learn Graph & Step-wise
& \textbf{0.0191} & \textbf{0.262} \\

\hline
\end{tabular}
\end{minipage}
\hfill
\begin{minipage}[t]{0.48\linewidth}
\centering
\textbf{(b) Acquisition utility}

\vspace{2pt}

\begin{tabular}{
p{3.5cm}<{\raggedright}
p{1.15cm}<{\centering}
p{1.35cm}<{\centering}
}
\hline
Acquisition score &RFR & DC
\\

\hline

Local reducibility $\ell_i$
& 0.0246
& 0.268
\\

Weighted analytic $U_i$
& 0.0157
& 0.256
\\
Unweighted $U_i$ ($\boldsymbol{W}_{k}=\boldsymbol{I}$)
& 0.0163
& 0.263
\\
Direct learned score $U_i^{ln}$
& 0.0174
& 0.268
\\
\textbf{Final calibrated $\widetilde U_i$}
& \textbf{0.0142}
& \textbf{0.236}
\\
\hline
\end{tabular}
\end{minipage}
\end{table}

\section{Conclusion}

We introduced GPARA for active grounding with frozen diffusion priors.
The method learns a graph surrogate over prediction residuals and
reuses its response structure for measurement refinement and candidate
evaluation. Under a matched surrogate, the theoretical analysis relates
weighted posterior-risk reduction to expected one-step refinement
benefit. The implementation combines cached prior predictions, sparse
conditioning, and bounded score calibration to avoid repeated diffusion
sample generation during acquisition. Experiments on radio-field reconstruction
and depth completion show improvements in both refinement and active
sensing. Ablations further show that step-wise structured refinement
outperforms terminal correction and that the weighted analytic acquisition
criterion is stronger than a learned-only scorer, supporting the role
of the proposed posterior-response mechanism beyond supervised policy
capacity. The current analysis is limited to the graph-Gaussian surrogate
and one-step acquisition, while nonlinear observations and longer-horizon
sensing remain important directions for future work.
\subsection*{Reproducibility Statement}

We provide the methodological and experimental details required to
reproduce GPARA. Section~\ref{sec:Method} specifies the graph-posterior
formulation, active-acquisition criterion, and the complete acquisition-reconstruction
pipeline summarized in Algorithm~\ref{alg:gpara}. Appendix~\ref{sec:Proofs-and-Numerical}
provides the numerical details underlying the method. The learned
graph predictor and residual scorer are generic function approximators
rather than architecture-specific components of the formulation; Appendix~\ref{sec:Network-Architectures--Offline}
specifies their functional interfaces and staged training objectives.
Section~\ref{sec:experiments} describes the datasets, sensing budgets,
evaluation metrics, baseline configurations, and experimental protocols.
All comparisons use the same frozen diffusion backbones. Source code and complete implementation configurations will be released publicly soon. 

\bibliography{references}
\bibliographystyle{iclr2027_conference}

\appendix

\section{Additional results}\label{sec:Additional-results}

\begin{table}[H]

\caption{Sensitivity to the number of reverse diffusion steps on RFR under
$3\%$ random observations. Results are full-region RMSE (mean $\pm$
standard deviation) over five paired seeds.}
\centering
\label{tab:rfr_diffusion_steps}
\resizebox{\textwidth}{!}{%
\renewcommand{\arraystretch}{1.08}
\begin{tabular}{lcccccc}
\hline
Method
& 5 steps
& 10 steps
& 20 steps
& 50 steps
& 100 steps
& 500 steps \\
\hline
DDPM
& 0.035 $\pm$ 0.0034
& 0.030 $\pm$ 0.0006
& 0.030 $\pm$ 0.0014
& 0.030 $\pm$ 0.0010
& 0.030 $\pm$ 0.0009
& 0.031 $\pm$ 0.0009 \\

DPS
& 0.031 $\pm$ 0.0023
& 0.028 $\pm$ 0.0006
& 0.028 $\pm$ 0.0013
& 0.027 $\pm$ 0.0009
& 0.024 $\pm$ 0.0003
& \textbf{0.018 $\pm$ 0.0002} \\

DDNM
& 0.033 $\pm$ 0.0031
& 0.028 $\pm$ 0.0005
& 0.028 $\pm$ 0.0013
& 0.027 $\pm$ 0.0009
& 0.027 $\pm$ 0.0007
& 0.027 $\pm$ 0.0009 \\

FPS
& 0.034 $\pm$ 0.0033
& 0.027 $\pm$ 0.0005
& 0.027 $\pm$ 0.0009
& 0.026 $\pm$ 0.0002
& 0.025 $\pm$ 0.0003
& 0.022 $\pm$ 0.0004 \\

DiffPIR
& 0.030 $\pm$ 0.0021
& 0.027 $\pm$ 0.0006
& 0.026 $\pm$ 0.0008
& 0.026 $\pm$ 0.0003
& 0.025 $\pm$ 0.0003
& 0.021 $\pm$ 0.0004 \\

GPARA
& \textbf{0.019 $\pm$ 0.0007}
& \textbf{0.017 $\pm$ 0.0002}
& \textbf{0.016 $\pm$ 0.0003}
& \textbf{0.016 $\pm$ 0.0002}
& \textbf{0.016 $\pm$ 0.0001}
& \textbf{0.016 $\pm$ 0.0002} \\
\hline
\end{tabular}%
}

\end{table}

\begin{table}[H]
\caption{Alignment between acquisition scores and empirical one-step reconstruction
benefit under shared random prefixes.}
\centering
\label{tab:utility_deployed_benefit}
\renewcommand\arraystretch{1.05}
\begin{tabular}{lcccc}
\hline
\multirow{2}{*}{Score} &\multicolumn{2}{c}{RFR} &\multicolumn{2}{c}{DC} \\
\cline{2-5}
& Spearman $\uparrow$ & Regret ($\times10^{-5}$) $\downarrow$ & Spearman $\uparrow$ & Regret ($\times10^{-4}$) $\downarrow$ \\
\hline
High-accuracy analytic & 0.42 $\pm$ 0.26 & 1.26 $\pm$ 1.45 & 0.19 $\pm$ 0.27 & 9.23 $\pm$ 15.60 \\
Deployed analytic      & 0.35 $\pm$ 0.27 & 1.14 $\pm$ 1.35 & 0.11 $\pm$ 0.25 & 9.51 $\pm$ 15.52 \\
Calibrated utility     & 0.40 $\pm$ 0.25 & 1.06 $\pm$ 1.25 & 0.15 $\pm$ 0.24 & \textbf{3.60 $\pm$ 4.97} \\
Direct learned scorer  & \textbf{0.44 $\pm$ 0.29} & \textbf{0.94 $\pm$ 1.28} & \textbf{0.23 $\pm$ 0.27} & 6.08 $\pm$ 14.75 \\
\hline
\end{tabular}

\end{table}

\paragraph{Sensitivity to reverse diffusion steps.}

Table~\ref{tab:rfr_diffusion_steps} evaluates sensitivity to the
reverse-step budget. \ProposedM reaches its best performance within
roughly $20$ steps, whereas several baselines continue to improve
markedly as more reverse steps are added; for example, DPS decreases
from $0.031$ RMSE at five steps to $0.018$ at $500$ steps. This
suggests that conventional inverse solvers rely more heavily on repeated
denoising and data-consistency updates to accumulate and propagate
measurement information. In contrast, the explicit graph-posterior
response in \ProposedM enables effective cross-location propagation
with substantially fewer reverse steps, making the method more compatible
with accelerated sampling. These results suggest that the proposed
explicit posterior propagation reduces sensitivity to the denoising-step
budget, although this ablation isolates reconstruction accuracy rather
than end-to-end wall-clock efficiency.

\paragraph{One-step diagnostic protocol.}

We evaluate whether each acquisition score predicts the empirical
benefit of the next measurement under shared random observation prefixes.
For each task, we use 20 test instances and construct three prefixes
per instance at approximately $10\%$, $50\%$, and $90\%$ of the
final sensing budget, yielding 60 matched prefix states. At each prefix
$\mathcal{S}$, all scorers evaluate the same randomly sampled candidate
set $\mathcal{C}^{\mathrm{rand}}$ with $|\mathcal{C}^{\mathrm{rand}}|=16$.
Each candidate is acquired independently and evaluated using four
paired runs of the complete Stage-II reconstruction. Its empirical
benefit is $g_{i}=E_{\mathrm{pre}}-E_{i}$, where $E_{\mathrm{pre}}$
and $E_{i}$ denote reconstruction MSE before and after measuring
candidate $i$, respectively, over a common evaluation region. For
candidate scores $s_{i}$ and the selected location $a=\arg\max_{i\in\mathcal{C}}s_{i}$,
we compute
\[
\rho_{s}=\mathrm{Corr}\left(\mathrm{rank}_{\mathcal{C}}(s_{i}),\mathrm{rank}_{\mathcal{C}}(g_{i})\right),\hspace*{0.3cm}\mathrm{Regret}=\max_{i\in\mathcal{C}}g_{i}-g_{a}=E_{a}-\min_{i\in\mathcal{C}}E_{i}.
\]
Here, $\mathrm{Corr}(\cdot)$ denotes Pearson correlation applied
to within-prefix ranks, with ties assigned average ranks. It measures
candidate-ranking agreement, while regret quantifies the selected
action’s excess MSE over the empirically best candidate. Both metrics
are computed per prefix and then aggregated.

The deployed analytic score uses the same numerical settings as the
main experiments. The high-accuracy analytic diagnostic uses PCG with
up to $96$ iterations and tolerance $10^{-7}$. Table~\ref{tab:utility_deployed_benefit}
shows that calibration improves the deployed analytic score on average,
increasing Spearman correlation from $0.35$ to $0.40$ on RFR and
from $0.11$ to $0.15$ on DC, while reducing regret. The direct learned
scorer achieves higher one-step rank on both tasks and the lowest
RFR regret, whereas the calibrated achieves the lowest DC regret.
These diagnostics assess only a single decision from a shared state;
they do not directly predict full-budget closed-loop performance.
Indeed, over complete acquisition, the calibrated policy achieves
lower final RMSE than the directly learned scorer.

These diagnostics evaluate a single next-measurement decision from
shared exogenous states, whereas Table~\ref{tab:mechanism_ablation1}
evaluates full-budget policy rollouts whose future states depend on
previous selections. The random-prefix distribution is also closer
to the randomly sampled observation states used during scorer training,
while closed-loop policies induce structured, policy-dependent states.
Moreover, each acquisition changes the posterior and therefore all
subsequent decisions. The analytic anchor explicitly recomputes marginal
value and redundancy from the updated posterior geometry at every
refresh, whereas the direct scorer must learn these effects from supervised
states. These differences provide plausible explanations for why the
calibrated policy achieves lower final RMSE than the direct scorer
despite not uniformly dominating it in the one-step diagnostics.

\section{Proofs and Numerical Details}\label{sec:Proofs-and-Numerical}

\subsection{Proof of Proposition~\ref{prop:Graph-Bayesian-residual}}\label{subsec:Proof-of-Proposition}

\paragraph{Positive definiteness of the posterior precision.}

For a symmetric graph with nonnegative conductances $c_{ij}=c_{ji}\geq0$
and every $\boldsymbol{v}\in\mathbb{R}^{M}$, the graph Laplacian
satisfies
\[
\boldsymbol{v}^{\mathsf{T}}\boldsymbol{L}_{g}\boldsymbol{v}=\frac{1}{2}\sum^{M}_{i=1}\sum^{M}_{j=1}c_{ij}(v_{i}-v_{j})^{2}\geq0.\tag{B.1}
\]
Hence for $\tau>0$ and $\lambda\geq0$, $\boldsymbol{L}_{g}\succeq0$
and $\boldsymbol{Q}_{g}\succeq\tau\boldsymbol{I}\succ0$. Since $\boldsymbol{R}_{\mathcal{S}}\succ0$,
for any nonzero $\boldsymbol{v}$, we have
\[
\begin{aligned}\boldsymbol{v}^{\mathsf{T}}\boldsymbol{A}_{\mathcal{S}}\boldsymbol{v} & =\tau\|\boldsymbol{v}\|^{2}_{2}+\lambda\boldsymbol{v}^{\mathsf{T}}\boldsymbol{L}_{g}\boldsymbol{v}+\left\Vert \boldsymbol{H}_{\mathcal{S}}\boldsymbol{v}\right\Vert ^{2}_{\boldsymbol{R}^{-1}_{\mathcal{S}}}\geq\tau\|\boldsymbol{v}\|^{2}_{2}>0.\end{aligned}
\tag{B.2}
\]
Thus $\boldsymbol{A}_{\mathcal{S}}\succeq\tau\boldsymbol{I}\succ0$.
This does not require the graph to be connected.

\paragraph{Conditional Gaussian posterior.}

The residual $\boldsymbol{x}_{0}=\widehat{\boldsymbol{x}}_{0|t}+\boldsymbol{\delta}$
and the observation model imply
\[
\boldsymbol{y}_{\mathcal{S}}\mid\boldsymbol{\delta},\widehat{\boldsymbol{x}}_{0|t},\boldsymbol{c}\sim\mathcal{N}\left(\boldsymbol{H}_{\mathcal{S}}\left(\widehat{\boldsymbol{x}}_{0|t}+\boldsymbol{\delta}\right),\boldsymbol{R}_{\mathcal{S}}\right).\tag{B.3}
\]
Combining this likelihood with $q_{\varPhi}(\boldsymbol{\delta}\mid\boldsymbol{c}_{g})=\mathcal{N}(\boldsymbol{0},\boldsymbol{Q}^{-1}_{g})$
gives
\[
\begin{aligned}q_{\varPhi}\left(\boldsymbol{\delta}\mid\widehat{\boldsymbol{x}}_{0|t},\boldsymbol{y}_{\mathcal{S}},\boldsymbol{c}_{g}\right) & \propto\exp\left(-\frac{1}{2}\boldsymbol{\delta}^{\mathsf{T}}\boldsymbol{Q}_{g}\boldsymbol{\delta}-\frac{1}{2}\left\Vert \boldsymbol{y}_{\mathcal{S}}-\boldsymbol{H}_{\mathcal{S}}\left(\widehat{\boldsymbol{x}}_{0|t}+\boldsymbol{\delta}\right)\right\Vert ^{2}_{\boldsymbol{R}^{-1}_{\mathcal{S}}}\right).\end{aligned}
\tag{B.4}
\]
Expanding the exponent and omitting terms independent of $\boldsymbol{\delta}$,
the negative log-density is
\begin{equation}
\frac{1}{2}\boldsymbol{\delta}^{\mathsf{T}}\boldsymbol{A}_{\mathcal{S}}\boldsymbol{\delta}-\boldsymbol{\delta}^{\mathsf{T}}\boldsymbol{H}^{\mathsf{T}}_{\mathcal{S}}\boldsymbol{R}^{-1}_{\mathcal{S}}\left(\boldsymbol{y}_{\mathcal{S}}-\boldsymbol{H}_{\mathcal{S}}\widehat{\boldsymbol{x}}_{0|t}\right)+\mathrm{const}.\tag{B.5}\label{eq:A5}
\end{equation}
Define $\boldsymbol{m}_{\mathcal{S},t}=\boldsymbol{A}^{-1}_{\mathcal{S}}\boldsymbol{H}^{\mathsf{T}}_{\mathcal{S}}\boldsymbol{R}^{-1}_{\mathcal{S}}\left(\boldsymbol{y}_{\mathcal{S}}-\boldsymbol{H}_{\mathcal{S}}\widehat{\boldsymbol{x}}_{0|t}\right)$.
Completing the square transforms (\ref{eq:A5}) into
\begin{equation}
\frac{1}{2}\left(\boldsymbol{\delta}-\boldsymbol{m}_{\mathcal{S},t}\right)^{\mathsf{T}}\boldsymbol{A}_{\mathcal{S}}\left(\boldsymbol{\delta}-\boldsymbol{m}_{\mathcal{S},t}\right)+\mathrm{const}.\tag{B.6}
\end{equation}
Because $\boldsymbol{A}_{\mathcal{S}}\succ0$, this expression defines
a proper Gaussian distribution. Consequently,
\[
q_{\varPhi}\left(\boldsymbol{\delta}\mid\widehat{\boldsymbol{x}}_{0|t},\boldsymbol{y}_{\mathcal{S}},\boldsymbol{c}_{g}\right)=\mathcal{N}\left(\boldsymbol{m}_{\mathcal{S},t},\boldsymbol{\Sigma}_{\mathcal{S}}\right),\hspace*{0.3cm}\boldsymbol{\Sigma}_{\mathcal{S}}=\boldsymbol{A}^{-1}_{\mathcal{S}}.\tag{B.7}
\]
This proves the posterior formulas and uniqueness in Proposition~\ref{prop:Graph-Bayesian-residual}.

\subsection{Proof of Proposition~\ref{prop:Proposition 2}}\label{subsec:Proof-of-Proposition-1}

Consider the current observation set $\mathcal{S}_{k}$ and a candidate
$i$. The prospective measurement is
\[
y_{i}=\boldsymbol{e}^{\mathsf{T}}_{i}\boldsymbol{x}_{0}+\eta_{i},\hspace*{0.3cm}\eta_{i}\sim\mathcal{N}(0,\sigma^{2}_{y}),\tag{B.8}
\]
where $\sigma^{2}_{y}>0$, $\boldsymbol{e}_{i}$ is the $i$th canonical
basis vector, and $\ensuremath{\eta_{i}}$ is independent of the residual
and noise.

\paragraph{Posterior-mean response to a new measurement.}

The prospective innovation can be written as
\[
\begin{aligned}\nu_{i} & =y_{i}-\boldsymbol{e}^{\mathsf{T}}_{i}\left(\widehat{\boldsymbol{x}}_{0|t}+\boldsymbol{m}_{\mathcal{S}_{k},t}\right)=\boldsymbol{e}^{\mathsf{T}}_{i}\left(\boldsymbol{\delta}_{t}-\boldsymbol{m}_{\mathcal{S}_{k},t}\right)+\eta_{i}.\end{aligned}
\tag{B.9}
\]
Its conditional moments are
\begin{equation}
\mathbb{E}_{q_{\varPhi}}\left[\nu_{i}\mid\mathcal{F}_{k}\right]=0,\hspace*{0.3cm}\mathbb{E}_{q_{\Phi}}\left[\nu^{2}_{i}\mid\mathcal{F}_{k}\right]=\sigma^{2}_{y}+\Sigma_{\mathcal{S}_{k},ii},\tag{B.10}\label{eq:A10}
\end{equation}
and
\begin{equation}
\mathbb{E}_{q_{\varPhi}}\left[\left(\boldsymbol{\delta}_{t}-\boldsymbol{m}_{\mathcal{S}_{k},t}\right)\nu_{i}\mid\mathcal{F}_{k}\right]=\boldsymbol{\Sigma}_{\mathcal{S}_{k}}\boldsymbol{e}_{i}.\tag{B.11}\label{eq:A11}
\end{equation}
Since the residual and the prospective measurement are jointly Gaussian,
conditioning on $y_{i}$ gives
\begin{equation}
\boldsymbol{m}_{\mathcal{S}_{k}\cup\{i\},t}=\boldsymbol{m}_{\mathcal{S}_{k},t}+\frac{\boldsymbol{\Sigma}_{\mathcal{S}_{k}}\boldsymbol{e}_{i}}{\sigma^{2}_{y}+\Sigma_{\mathcal{S}_{k},ii}}\nu_{i}.\tag{B.12}\label{eq:A12}
\end{equation}
Hence the target posterior-mean change equals the residual posterior-mean
change:
\begin{equation}
\Delta\widehat{\boldsymbol{x}}_{i}=\boldsymbol{m}_{\mathcal{S}_{k}\cup\{i\},t}-\boldsymbol{m}_{\mathcal{S}_{k},t}=\boldsymbol{u}_{i}\nu_{i},\hspace*{0.3cm}\boldsymbol{u}_{i}=\frac{\boldsymbol{\Sigma}_{\mathcal{S}_{k}}\boldsymbol{e}_{i}}{\sigma^{2}_{y}+\Sigma_{\mathcal{S}_{k},ii}}.\tag{B.13}\label{eq:A13}
\end{equation}

\paragraph{Covariance-risk reduction.}

The additional independent point measurement updates the precision
\[
\boldsymbol{A}_{\mathcal{S}_{k}\cup\{i\}}=\boldsymbol{A}_{\mathcal{S}_{k}}+\sigma^{-2}_{y}\boldsymbol{e}_{i}\boldsymbol{e}^{\mathsf{T}}_{i}.\tag{B.14}
\]
Applying the rank-one inverse identity yields
\[
\boldsymbol{\Sigma}_{\mathcal{S}_{k}\cup\{i\}}=\boldsymbol{\Sigma}_{\mathcal{S}_{k}}-\frac{\boldsymbol{\Sigma}_{\mathcal{S}_{k}}\boldsymbol{e}_{i}\boldsymbol{e}^{\mathsf{T}}_{i}\boldsymbol{\Sigma}_{\mathcal{S}_{k}}}{\sigma^{2}_{y}+\Sigma_{\mathcal{S}_{k},ii}}.\tag{B.15}
\]
Using the same current weighting matrix $\boldsymbol{W}_{k}$ on both
sides of the hypothetical update,
\[
\begin{aligned}U_{i}(\mathcal{S}_{k}) & =\mathrm{tr}\left(\boldsymbol{W}_{k}\left(\boldsymbol{\Sigma}_{\mathcal{S}_{k}}-\boldsymbol{\Sigma}_{\mathcal{S}_{k}\cup\{i\}}\right)\right)=\frac{\boldsymbol{e}^{\mathsf{T}}_{i}\boldsymbol{\Sigma}_{\mathcal{S}_{k}}\boldsymbol{W}_{k}\boldsymbol{\Sigma}_{\mathcal{S}_{k}}\boldsymbol{e}_{i}}{\sigma^{2}_{y}+\Sigma_{\mathcal{S}_{k},ii}}=\left(\sigma^{2}_{y}+\Sigma_{\mathcal{S}_{k},ii}\right)\|\boldsymbol{u}_{i}\|^{2}_{\boldsymbol{W}_{k}}.\end{aligned}
\tag{B.16}
\]
The second equality follows from cyclic invariance of the trace. Since
$\boldsymbol{W}_{k}\succeq0$, the utility is nonnegative. Finally,
$\boldsymbol{u}_{i}$ is fixed given the current state. Combining
(\ref{eq:A10}) and (\ref{eq:A13}),
\[
\begin{aligned}\mathbb{E}_{q_{\varPhi}}\left[\|\Delta\widehat{\boldsymbol{x}}_{i}\|^{2}_{\boldsymbol{W}_{k}}\mid\mathcal{F}_{k}\right]=\mathbb{E}_{q_{\mathit{\Phi}}}\left[\nu^{2}_{i}\mid\mathcal{F}_{k}\right]\|\boldsymbol{u}_{i}\|^{2}_{\boldsymbol{W}_{k}}=\left(\sigma^{2}_{y}+\Sigma_{\mathcal{S}_{k},ii}\right)\|\boldsymbol{u}_{i}\|^{2}_{\boldsymbol{W}_{k}}.\end{aligned}
\tag{B.17}
\]
The identity establishes that the covariance reduction and the expected
squared posterior-mean change describe the same one-step effect under
the surrogate. This proves Proposition~\ref{prop:Proposition 2}.

\subsection{Proof of Corollary~\ref{cor:Ranking Consistency}}\label{subsec:Proof-of-Corollary}

For the unclipped comparison in Corollary~\ref{cor:Ranking Consistency},
the two uniformly scaled estimates are
\[
\widetilde{\boldsymbol{x}}_{0|t,\mathcal{S}_{k}}=\widehat{\boldsymbol{x}}_{0|t}+\rho_{g}\boldsymbol{m}_{\mathcal{S}_{k},t},\hspace*{0.3cm}\widetilde{\boldsymbol{x}}_{0|t,\mathcal{S}_{k}\cup\{i\}}=\widehat{\boldsymbol{x}}_{0|t}+\rho_{g}\boldsymbol{m}_{\mathcal{S}_{k}\cup\{i\},t}.\tag{B.18}
\]

By (\ref{eq:A12}), $\widetilde{\boldsymbol{x}}_{0|t,\mathcal{S}_{k}\cup\{i\}}=\widetilde{\boldsymbol{x}}_{0|t,\mathcal{S}_{k}}+\rho_{g}\boldsymbol{u}_{i}\nu_{i}$.
Since $\boldsymbol{x}_{0}-\widetilde{\boldsymbol{x}}_{0|t,\mathcal{S}_{k}}=\boldsymbol{\delta}_{t}-\rho_{g}\boldsymbol{m}_{\mathcal{S}_{k},t}$,
expanding the weighted squared-error difference gives
\begin{equation}
\begin{aligned}\Delta_{i}(\rho_{g}) & =\left\Vert \boldsymbol{\delta}_{t}-\rho_{g}\boldsymbol{m}_{\mathcal{S}_{k},t}\right\Vert ^{2}_{\boldsymbol{W}_{k}}-\left\Vert \boldsymbol{\delta}_{t}-\rho_{g}\boldsymbol{m}_{\mathcal{S}_{k},t}-\rho_{g}\boldsymbol{u}_{i}\nu_{i}\right\Vert ^{2}_{\boldsymbol{W}_{k}}\\
 & =2\rho_{g}\nu_{i}\boldsymbol{u}^{\mathsf{T}}_{i}\boldsymbol{W}_{k}\left(\boldsymbol{\delta}_{t}-\rho_{g}\boldsymbol{m}_{\mathcal{S}_{k},t}\right)-\rho^{2}_{g}\nu^{2}_{i}\|\boldsymbol{u}_{i}\|^{2}_{\boldsymbol{W}_{k}}.
\end{aligned}
\tag{B.19}\label{eq:A19}
\end{equation}
The scaled estimate is generally not the surrogate posterior mean
when $\rho_{g}\neq1$. To account for this distinction, decompose
its residual error as
\[
\boldsymbol{\delta}_{t}-\rho_{g}\boldsymbol{m}_{\mathcal{S}_{k},t}=\left(\boldsymbol{\delta}_{t}-\boldsymbol{m}_{\mathcal{S}_{k},t}\right)+(1-\rho_{g})\boldsymbol{m}_{\mathcal{S}_{k},t}.\tag{B.20}
\]
Using the zero innovation mean in (\ref{eq:A10}) and the cross-moment
in (\ref{eq:A11}),
\begin{equation}
\begin{aligned} & \mathbb{E}_{q_{\mathit{\Phi}}}\left[\left(\boldsymbol{\delta}_{t}-\rho_{g}\boldsymbol{m}_{\mathcal{S}_{k},t}\right)\nu_{i}\mid\mathcal{F}_{k}\right]\\
 & \quad=\mathbb{E}_{q_{\mathit{\Phi}}}\left[\left(\boldsymbol{\delta}_{t}-\boldsymbol{m}_{\mathcal{S}_{k},t}\right)\nu_{i}\mid\mathcal{F}_{k}\right]+(1-\rho_{g})\boldsymbol{m}_{\mathcal{S}_{k},t}\mathbb{E}_{q_{\mathit{\Phi}}}\left[\nu_{i}\mid\mathcal{F}_{k}\right]\\
 & \quad=\boldsymbol{\Sigma}_{\mathcal{S}_{k}}\boldsymbol{e}_{i}.
\end{aligned}
\tag{B.21}\label{eq:A21}
\end{equation}
Taking the conditional expectation of (\ref{eq:A19}) therefore yields
\[
\begin{aligned} & \mathbb{E}_{q_{\mathit{\Phi}}}\left[\Delta_{i}(\rho_{g})\mid\mathcal{F}_{k}\right]=2\rho_{g}\boldsymbol{u}^{\mathsf{T}}_{i}\boldsymbol{W}_{k}\boldsymbol{\Sigma}_{\mathcal{S}_{k}}\boldsymbol{e}_{i}-\rho^{2}_{g}\left(\sigma^{2}_{y}+\Sigma_{\mathcal{S}_{k},ii}\right)\|\boldsymbol{u}_{i}\|^{2}_{\boldsymbol{W}_{k}}\\
 & \quad=(2\rho_{g}-\rho^{2}_{g})\left(\sigma^{2}_{y}+\Sigma_{\mathcal{S}_{k},ii}\right)\|\boldsymbol{u}_{i}\|^{2}_{\boldsymbol{W}_{k}}=(2\rho_{g}-\rho^{2}_{g})U_{i}(\mathcal{S}_{k}),
\end{aligned}
\tag{B.22}
\]
where the second equality uses $\boldsymbol{\Sigma}_{\mathcal{S}_{k}}\boldsymbol{e}_{i}=\left(\sigma^{2}_{y}+\Sigma_{\mathcal{S}_{k},ii}\right)\boldsymbol{u}_{i}.$
For a common $0<\rho_{g}<2$, the factor $2\rho_{g}-\rho^{2}_{g}=\rho_{g}(2-\rho_{g})$
is positive and independent of the candidate. Hence
\[
\begin{aligned}\arg\max_{i\in\mathcal{C}_{k}}U_{i}(\mathcal{S}_{k})=\arg\max_{i\in\mathcal{C}_{k}}\mathbb{E}_{q_{\mathit{\Phi}}}\left[\Delta_{i}(\rho_{g})\mid\mathcal{F}_{k}\right],\end{aligned}
\tag{B.23}
\]
with equality understood as equality of maximizing sets when ties
occur. This proves Corollary~\ref{cor:Ranking Consistency}.

\subsection{A Bounded-Score Consequence}
\begin{lem}
\label{lem:(Bounded)}(Bounded Departure from the Analytic Anchor).
Let $\bar{U}_{i}$ denote the normalized analytic utility, and let
its bounded calibrated score be $\widetilde{U}_{i}=\bar{U}_{i}+c_{i}$
with $|c_{i}|\le\beta_{\mathrm{acq}}$. Define
\[
i^{\star}\in\arg\max_{i}\bar{U}_{i},\hspace*{0.3cm}\hat{i}\in\arg\max_{i}\widetilde{U}_{i}.
\]
Then $0\le\bar{U}_{i^{\star}}-\bar{U}_{\hat{i}}\le2\beta_{\mathrm{acq}}.$
Consequently, if the analytic margin satisfies
\[
\bar{U}_{i^{\star}}-\max_{j\neq i^{\star}}\bar{U}_{j}>2\beta_{\mathrm{acq}},\tag{B.24}
\]
then the analytic maximizer $i^{\star}$ remains the unique maximizer
after calibration.
\end{lem}
\begin{proof}
By the optimality of $\widehat{i}$,
\[
\bar{U}_{\widehat{i}}+c_{\widehat{i}}\ge\bar{U}_{i^{\star}}+c_{i^{\star}}.
\]
Therefore,
\[
\bar{U}_{i^{\star}}-\bar{U}_{\widehat{i}}\le c_{\widehat{i}}-c_{i^{\star}}\le|c_{\widehat{i}}|+|c_{i^{\star}}|\le2\beta_{\mathrm{acq}}.
\]
If the analytic margin exceeds $2\beta_{\mathrm{acq}}$, no other
candidate can overcome this gap under an admissible residual correction,
proving the second statement.
\end{proof}

Lemma~\ref{lem:(Bounded)} formalizes the role of bounded residual
calibration. The learned scorer may refine the ordering of candidates
whose analytic utilities are close, but it cannot arbitrarily overturn
a sufficiently strong analytic preference. The final acquisition rule
therefore remains anchored to the theoretically motivated posterior-risk
criterion while retaining limited flexibility to compensate for surrogate
mismatch.

\subsection{Matrix-Free Graph Solves and Covariance Probing}\label{subsec:Matrix-Free-Graph-Solves}

GPARA does not explicitly form the dense posterior covariance $\boldsymbol{\Sigma}_{\mathcal{S}}=\boldsymbol{A}^{-1}_{\mathcal{S}}.$
Instead, posterior refinement and covariance-dependent acquisition
statistics are evaluated through sparse linear solves with the posterior
precision $\boldsymbol{A}_{\mathcal{S}}$. For independent noise observations,
write $\mathbf{A}_{\mathcal{S}}=\mathbf{D}_{A}-\lambda\mathbf{C}_{g},$
where
\[
[\mathbf{D}_{A}]_{ii}=\tau+\lambda d_{i}+\sigma^{-2}_{y}\mathbf{1}_{\{i\in\mathcal{S}\}},\hspace*{0.3cm}d_{i}=\sum_{j}c_{ij}.\tag{B.25}
\]
For a reference prediction $\boldsymbol{\mu}$, define $\mathbf{f}_{\mathcal{S}}=\boldsymbol{H}^{\mathrm{T}}_{\mathcal{S}}\boldsymbol{R}^{-1}_{\mathcal{S}}\left(\boldsymbol{y}_{\mathcal{S}}-\boldsymbol{H}_{\mathcal{S}}\boldsymbol{\mu}\right).$
The posterior-mean correction $\boldsymbol{m}_{\mathcal{S}}=\boldsymbol{A}^{-1}_{\mathcal{S}}\mathbf{f}_{\mathcal{S}}$
is approximated by fixed-depth Jacobi iterations:
\[
\boldsymbol{m}^{(j+1)}=\mathbf{D}^{-1}_{A}\left(\mathbf{f}_{\mathcal{S}}+\lambda\mathbf{C}_{g}\boldsymbol{m}^{(j)}\right).\tag{B.26}
\]
Because $\ensuremath{\tau>0}$ and $\ensuremath{c_{ij}\ge0}$ the
Jacobi iteration is contractive under the point-sensing model, with
\[
q=\max_{i}\frac{\lambda d_{i}}{\tau+\lambda d_{i}+\sigma^{-2}_{y}\mathbf{1}_{\{i\in\mathcal{S}\}}}<1.\tag{B.27}
\]
Thus, Jacobi converges to the exact mean as its iteration depth increases.
However, we use fixed iterations and treat the resulting correction
as a finite-depth approximation rather than an exact solve.

For active acquisition, the utility requires only $\Sigma_{k,ii}$
and $\left(\boldsymbol{\Sigma}_{k}\boldsymbol{W}_{k}\boldsymbol{\Sigma}_{k}\right)_{ii}$.
We estimate these quantities using $P$ independent Rademacher probes
as
\[
\mathbf{z}^{(p)}\in\{-1,+1\}^{M},\hspace*{0.3cm}\mathbb{E}[\mathbf{z}^{(p)}\mathbf{z}^{(p)\mathsf{T}}]=\boldsymbol{I}.\tag{B.28}
\]
For each probe, we solve the linear systems by the preconditioned
conjugate gradient (PCG) method
\begin{equation}
\boldsymbol{A}_{k}\mathbf{v}^{(p)}=\mathbf{z}^{(p)},\hspace*{0.3cm}\boldsymbol{A}_{k}\mathbf{s}^{(p)}=\boldsymbol{W}^{1/2}_{k}\mathbf{z}^{(p)}\tag{B.29}\label{eq:B29}
\end{equation}
using preconditioned conjugate gradients with the diagonal Jacobi
preconditioner $\boldsymbol{P}_{k}=\mathrm{diag}(\boldsymbol{A}_{k}).$

With exact solves,
\[
\mathbb{E}\left[z^{(p)}_{i}v^{(p)}_{i}\right]=\Sigma_{k,ii},\hspace*{0.3cm}\mathbb{E}\left[\left(s^{(p)}_{i}\right)^{2}\right]=\left(\boldsymbol{\Sigma}_{k}\boldsymbol{W}_{k}\boldsymbol{\Sigma}_{k}\right)_{ii}.\tag{B.30}
\]
Accordingly, we evaluate the numerical acquisition utility as
\begin{equation}
\widehat{U}_{i}=\frac{\widehat{n}_{i}}{\sigma^{2}_{y}+\max(\widehat{d}_{i},[\mathbf{D}_{A}]^{-1}_{ii},\varepsilon_{d})},\hspace*{0.3cm}\widehat{d}_{i}=\frac{1}{P}\sum^{P}_{p=1}z^{(p)}_{i}v^{(p)}_{i},\hspace*{0.3cm}\widehat{n}_{i}=\frac{1}{P}\sum^{P}_{p=1}\left(s^{(p)}_{i}\right)^{2}.\tag{B.31}\label{eq:B31}
\end{equation}
The two raw moment estimators are unbiased under exact solves, but
their stabilized ratio is generally biased at finite $P$ because
of probe variance, correlation between numerator and denominator,
the nonlinear ratio, and truncated iterative solves. We therefore
treat $\widehat{U}_{i}$ as a numerical approximation of the analytic
utility rather than an unbiased estimator.

For both tasks, we use prior samples $L=4$, graph precision $\tau=0.0625$,
$\varepsilon_{d}=10^{-12}$, $24$ Jacobi iterations, and PCG with
at most 24 iterations and a relative tolerance of $10^{-5}$. RFR
uses $P=16$, $\lambda=0.8$, and $\sigma_{y}=10^{-4}$, whereas DC
uses $P=4$, $\lambda=1.0$, and $\sigma_{y}=2\times10^{-3}$.

\subsection{Construction of the Observation-Set-Adaptive Risk Matrix}\label{subsec:risk-weighting matrix}

The graph posterior and the risk matrix capture two complementary
aspects of active acquisition. The graph posterior $\boldsymbol{\Sigma}_{k}$
determines the spatial response induced by a prospective measurement:
its entry $\Sigma_{k,ji}$ controls how strongly an innovation observed
at candidate $i$ influences variable $j$. In contrast, $\boldsymbol{W}_{k}$
emphasizes locations that remain less resolved under the current observation
set. Locations exhibiting larger disagreement among the graph-conditioned
samples receive larger weights. Consequently, a candidate is valuable
not merely when it influences many variables, but when its response
reaches regions with relatively high remaining disagreement.

Let $\{\widehat{\boldsymbol{x}}^{(l)}_{0}\}^{L}_{l=1}$, $\widehat{\boldsymbol{x}}^{(l)}_{0}\sim p_{\theta}(\boldsymbol{x}_{0}\mid\boldsymbol{c})$
denote the $L$ cached samples produced by the frozen diffusion prior.
Such an ensemble is generated once before active acquisition and cached
thereafter.

At refresh $k$, the observations $(\mathcal{S}_{k},\boldsymbol{y}_{\mathcal{S}_{k}})$
are applied to $\widehat{\boldsymbol{x}}^{(l)}_{0}$ by Eq.~(\ref{eq:correction}).
Let $\widetilde{\boldsymbol{x}}^{(l)}_{0|\mathcal{S}_{k}}$ be the
resulting graph-conditioned sample after correction. The ensemble
mean and empirical variance are
\[
\widehat{\boldsymbol{\mu}}_{\mathcal{S}_{k}}=\frac{1}{L}\sum^{L}_{l=1}\widetilde{\boldsymbol{x}}^{(l)}_{0|\mathcal{S}_{k}},\hspace*{0.3cm}\widehat{\boldsymbol{v}}_{\mathcal{S}_{k}}=\frac{1}{L-1}\sum^{L}_{l=1}\left(\widetilde{\boldsymbol{x}}^{(l)}_{0|\mathcal{S}_{k}}-\widehat{\boldsymbol{\mu}}_{\mathcal{S}_{k}}\right)^{2}.\tag{B.32}
\]
The average empirical variance over unobserved admissible locations
is given by
\[
\overline{v}_{\mathcal{S}_{k}}=\frac{\sum_{j\in\mathcal{C}_{k}}{\displaystyle \widehat{v}_{\mathcal{S}_{k},j}}}{\max\!\left(|\mathcal{C}_{k}|,1\right)},\tag{B.33}
\]
where ${\displaystyle \widehat{v}_{\mathcal{S}_{k},j}}$ is the value
at location $j$. We then construct $\boldsymbol{W}_{k}=\mathrm{diag}(w_{k,1},...,w_{k,M})\succeq0$
with
\[
\begin{aligned}\ensuremath{w_{k,i}=\mathbf{1}_{\{i\in\mathcal{C}_{k}\}}\times\left(\varsigma+\frac{\Delta w}{r}\Pi_{[0,r]}\left(\frac{\widehat{v}_{\mathcal{S}_{k},i}}{\max(\overline{v}_{\mathcal{S}_{k}},\varepsilon_{v})}\right)\right)}\end{aligned}
.\tag{B.34}
\]
Here, the hyperparameters $\varsigma>0$, $\Delta w\geq0$, $r>0$,
and $\varepsilon_{v}>0$ specify the baseline weight, maximum additional
weight, relative-variance cap, and numerical floor, respectively.
In this work, we set $\varsigma=0.25$, $\Delta w=0.75$, $r=4$,
$\varepsilon_{v}=10^{-8}$ and $L=4.$

Under the unclipped linear graph update, the measurement-dependent
correction contains a common additive shift across all cached samples,
which cancels when computing the centered ensemble variance. Therefore,
$\boldsymbol{W}_{k}$ is primarily adaptive to the observation set
and graph-posterior response rather than directly to the realized
measurement values. The small diffusion ensemble is generated only
once and cached. The cached ensemble is reconditioned after each acquisition
refresh, allowing the risk emphasis to adapt to the evolving observation
set without regenerating diffusion samples.

\subsection{Batched Acquisition Between Posterior Refreshes}\label{subsec:Batched-Acquisition}

To reduce posterior-refresh cost, the implementation can use batched
acquisition as a computational approximation to the single-point rule
in Section~\ref{sec:Method}. At refresh $k$, we select at most
$B=8$ measurements, i.e., $B_{k}=\min\{B,\Gamma-|\mathcal{S}_{k}|\}$
from the current candidate set $\mathcal{C}_{k}$. The graph posterior,
risk matrix, and calibrated scores are computed once and remain fixed
while constructing the batch.

To discourage redundant selections, we reuse the weighted probe solutions
$\mathbf{s}^{(p)}$ from (\ref{eq:B29}). For candidates $i$ and
$j$, we define the response-overlap proxy
\[
\widehat{\kappa}_{ij}=\Pi_{[0,1]}\left(\frac{\left(\sum^{P}_{p=1}s^{(p)}_{i}s^{(p)}_{j}\right)^{2}}{\max\left(\max(P\widehat{n}_{i},10^{-12})\max(P\widehat{n}_{j},10^{-12}),10^{-12}\right)}\right),\tag{B.35}
\]
where $\widehat{n}_{i}$ is defined in (\ref{eq:B31}). Hence, no
additional linear solves are required beyond those already used for
acquisition-utility estimation. Let
\[
D_{k}=\max\left(\max_{i\in\mathcal{C}_{k}}\widetilde{U}_{i}(\mathcal{S}_{k})-\min_{i\in\mathcal{C}_{k}}\widetilde{U}_{i}(\mathcal{S}_{k}),1\right).\tag{B.36}
\]
Starting with $\mathcal{B}^{(0)}_{k}=\varnothing$, the batch is constructed
greedily as
\[
\begin{aligned}i_{k,\ell} & \in\arg\max_{i\in\mathcal{C}_{k}\setminus\mathcal{B}^{(\ell-1)}_{k}}\left(\widetilde{U}_{i}(\mathcal{S}_{k})-0.5D_{k}\sum_{j\in\mathcal{B}^{(\ell-1)}_{k}}\widehat{\kappa}_{ij}\right),\\
\mathcal{B}^{(\ell)}_{k} & =\mathcal{B}^{(\ell-1)}_{k}\cup\{i_{k,\ell}\},\hspace*{0.3cm}\ell=1,...,B_{k}.
\end{aligned}
\tag{B.37}
\]
Thus, the first point maximizes the calibrated utility, while subsequent
points are penalized according to their cumulative response overlap
with already selected batch members. After the batch $\mathcal{B}_{k}=\mathcal{B}^{(B_{k})}_{k}$
is determined, measurements are acquired and the observation set is
updated as
\[
\mathcal{S}_{k+1}=\mathcal{S}_{k}\cup\mathcal{B}_{k}.
\]
The graph posterior, risk matrix, and acquisition scores are then
recomputed for the next refresh. No newly acquired value is used during
construction of the current batch. Note that the single-point utility
remains the underlying acquisition criterion of GPARA, while batching
serves only as a computational approximation rather than an exact
joint batch-risk optimization.

\section{Network Architectures and Training}\label{sec:Network-Architectures--Offline}

The graph predictor $G_{\phi}$ and residual scorer $r_{\psi}$ are
learned function approximators whose architectures are not intrinsic
to the GPARA formulation. GPARA requires only that $G_{\phi}$ produce
valid nonnegative graph conductances from the available context and
that $r_{\psi}$ output candidate-wise bounded score corrections.
We therefore specify these interfaces and mappings below, while using
standard lightweight architectures for their implementation. The two
components are trained sequentially. With the diffusion backbone frozen,
$G_{\phi}$ is first optimized through differentiable graph conditioning.
$G_{\phi}$ is then frozen and $r_{\psi}$ is trained using privileged
supervision constructed from training targets. All networks are trained
offline and remain fixed during deployment.

\subsection{Network interfaces.}

The graph predictor uses a lightweight UNet to generate four directional
edge-logit maps from $\boldsymbol{c}_{g}$:
\[
G_{\phi}(\boldsymbol{c}_{g})\longrightarrow\mathbf{Z}_{\phi}=\{z_{i,\xi}\}_{i\in\mathcal{V},\xi\in\mathcal{D}}\in\mathbb{R}^{M\times4},\hspace*{0.3cm}\mathcal{D}=\{(0,1),(1,0),(1,1),(1,-1)\},\tag{C.1}
\]
corresponding to right, down, down-right, and down-left canonical
edges. Each undirected local edge is predicted once, and its reverse
direction is introduced deterministically when assembling the symmetric
graph. For $j=i+\xi$, let $m_{i,\xi}\in\{0,1\}$ be the valid-edge
mask. The final conductance is
\[
c_{ij}=c_{ji}=m_{i,\xi}\frac{2\mathrm{sigmoid}(z_{i,\xi})}{\|\xi\|_{2}},\hspace*{0.3cm}\xi\in\mathcal{D}.\tag{C.2}
\]
Thus, horizontal and vertical conductances lie in $[0,2]$, whereas
diagonal conductances are additionally scaled by $1/\sqrt{2}$. Invalid
or out-of-domain edges and self-connections are assigned zero weight.
The resulting adjacency matrix $\boldsymbol{C}_{g}=[c_{ij}]$ is sparse,
symmetric, and nonnegative. Four canonical edge maps therefore represent
an undirected graph with up to eight neighbors per interior node. 

The acquisition scorer independently outputs one residual score per
target location,
\[
R_{\psi}(\mathbf{z}_{\mathcal{S}_{k}})=\left[r_{\psi}(1;\mathbf{z}_{\mathcal{S}_{k}}),...,r_{\psi}(M;\mathbf{z}_{\mathcal{S}_{k}})\right]^{\mathrm{T}}\in\mathbb{R}^{M}.\tag{C.3}
\]
The scorer combines a convolutional encoder-decoder, four graph-message-passing
blocks, and a budget-conditioning MLP. The input $\mathbf{z}_{\mathcal{S}_{k}}$
captures the available acquisition state, combining task context,
accumulated observation information, graph-posterior statistics, the
analytic acquisition feature, and the remaining sensing budget.

\subsection{Learning the graph through posterior refinement.}

For each training example, we randomly sample a reverse step $t$,
obtain the clean estimate $\widehat{\boldsymbol{x}}_{0|t}$ from the
frozen diffusion model, and sample a sparse observation set $\mathcal{S}$.
The graph predictor constructs $\boldsymbol{C}_{g}$ and obtains $\boldsymbol{A}_{\mathcal{S}}$.
We then apply the same differentiable Jacobi conditioner used at deployment
to obtain the residual correction and refined prediction $\widetilde{\boldsymbol{x}}_{0|t}$
through Eq.~(\ref{eq:correction}). Gradients are propagated through
both the graph construction and the unrolled conditioner to $G_{\phi}$.

Let $\mathcal{V}_{\mathrm{val}}\subseteq\{1,\ldots,M\}$ denote the
valid spatial region, and $\mathcal{U}=\mathcal{V}_{\mathrm{val}}\setminus\mathcal{S}$
is the unobserved region. For any subset $\mathcal{R}\ensuremath{\subseteq\mathcal{V}_{\mathrm{val}}}$,
define the masked reconstruction error as
\[
\ensuremath{E_{\mathcal{\mathcal{R}}}(\boldsymbol{\hat{x}}_{0},\boldsymbol{x}_{0})=\frac{{\displaystyle \sum_{i\in\mathcal{\mathcal{R}}}\left|\boldsymbol{\hat{x}}_{0,i}-\boldsymbol{x}_{0,i}\right|^{2}}}{\max(|\mathcal{\mathcal{R}}|,1)}.}\tag{C.4}
\]
The graph-training objective is formulated as
\[
\mathcal{L}_{\mathrm{graph}}=E_{\mathcal{V}_{\mathrm{val}}}(\widetilde{\boldsymbol{x}}_{0|t},\boldsymbol{x}_{0})+\lambda_{u}E_{\mathcal{U}}(\widetilde{\boldsymbol{x}}_{0|t},\boldsymbol{x}_{0})+\lambda_{s}E_{\mathcal{S}}(\widetilde{\boldsymbol{x}}_{0|t},\boldsymbol{y}_{\mathcal{S}})+\lambda_{h}\mathcal{L}_{\mathrm{hinge}}+\lambda_{a}\mathcal{L}_{\mathrm{aff}}.\tag{C.5}
\]
The first three terms supervise full-region reconstruction, propagation
to unobserved locations, and observation consistency, respectively.
To discourage graph refinement from degrading the frozen prediction,
we use
\[
\mathcal{L}_{\mathrm{hinge}}=\mathrm{max}\left(E_{\mathcal{U}}\left(\widetilde{\boldsymbol{x}}_{0|t},\boldsymbol{x}_{0}\right)-E_{\mathcal{U}}\left(\widehat{\boldsymbol{x}}_{0|t},\boldsymbol{x}_{0}\right),0\right).\tag{C.6}
\]
We additionally supervise local residual compatibility. Let $\boldsymbol{\delta}_{t}=\boldsymbol{x}_{0}-\widehat{\boldsymbol{x}}_{0|t}$.
For each valid canonical edge $(i,j)\in\mathcal{E}$,
\[
a^{\mathrm{res}}_{ij}=\exp\left(-\frac{|\boldsymbol{\delta}_{t,i}-\boldsymbol{\delta}_{t,j}|}{s_{\mathrm{res}}}-\frac{|\boldsymbol{x}_{0,i}-\boldsymbol{x}_{0,j}|}{s_{x}}\right),\tag{C.7}
\]
where $s_{\mathrm{res}}>0$ and $s_{x}>0$ are fixed scales. Let $p^{\phi}_{ij}=\mathrm{sigmoid}(z_{i,\xi})$
for $j=i+\xi$. The affinity loss acts on these pre-rescaling edge
probabilities:
\[
\mathcal{L}_{\mathrm{aff}}=-\frac{1}{|\mathcal{E}|}\sum_{(i,j)\in\mathcal{E}}\left[a^{\mathrm{res}}_{ij}\log p^{\phi}_{ij}+(1-a^{\mathrm{res}}_{ij})\log(1-p^{\phi}_{ij})\right].\tag{C.8}
\]
Each undirected edge appears once in $\mathcal{E}$. The affinity
loss therefore supervises the pre-rescaling edge compatibility, while
the reconstruction losses optimize the final conductances end-to-end
through the differentiable graph-Bayesian conditioner. Importantly,
these conductances couple prediction-residual corrections rather than
directly enforcing similarity between reconstructed target values.

In this work, we set $\lambda_{u}=2$, $\lambda_{s}=0.1$, $\lambda_{h}=0.1$,
$\lambda_{a}=0.01$, $s_{\mathrm{res}}=0.05$ and $s_{x}=0.08$.

\subsection{Learning the Acquisition Residual Scorer}

For each training example, an observation set is sampled by $\mathcal{S}_{k}\sim\mathrm{Uniform}\left(\left\{ \mathcal{A}\subseteq\mathcal{V}_{\mathrm{val}}:|\mathcal{A}|<\Gamma\right\} \right)$.
The graph conditioner assimilates $(\mathcal{S}_{k},\mathbf{y}_{\mathcal{S}_{k}})$,
after which we compute the analytic utilities $U_{i}$ for every $i\in\mathcal{C}_{k}$
and assemble $\mathbf{z}_{\mathcal{S}_{k}}$. Consequently, revisiting
the same training sample can produce different examples through different
observation sets.

Let $\widehat{\boldsymbol{\mu}}_{\mathcal{S}_{k}}$ denote the current
graph-conditioned estimate and the error $\boldsymbol{e}=\boldsymbol{x}_{0}-\widehat{\boldsymbol{\mu}}_{\mathcal{S}_{k}}$.
Using the graph response $\boldsymbol{u}_{i}=\boldsymbol{\Sigma}_{\mathcal{S}_{k}}\boldsymbol{e}_{i}/(\sigma^{2}_{y}+\Sigma_{\mathcal{S}_{k},ii})$,
the training-only linear-response teacher evaluates
\[
\begin{aligned}\widehat{\Delta}_{i} & =2\rho_{g}\nu^{\mathrm{tr}}_{i}\boldsymbol{u}^{\top}_{i}\mathbf{W}_{\mathrm{tr}}\boldsymbol{e}-\rho^{2}_{g}(\nu^{\mathrm{tr}}_{i})^{2}\boldsymbol{u}^{\top}_{i}\mathbf{W}_{\mathrm{tr}}\boldsymbol{u}_{i}.\end{aligned}
\tag{C.9}
\]
Here, $\mathbf{W}_{\mathrm{tr}}=\mathrm{diag}(\mathbf{1}_{\{1\in\mathcal{V}_{\mathrm{val}}\}},...,\mathbf{1}_{\{M\in\mathcal{V}_{\mathrm{val}}\}})$,
and $\nu^{\mathrm{tr}}_{i}=\boldsymbol{x}_{0,i}-(\overline{\boldsymbol{x}}_{0}+\overline{\boldsymbol{m}}_{\mathcal{S}_{k}})_{i},$
where $\overline{\boldsymbol{x}}_{0}$ is the original cached prior
mean and $\overline{\boldsymbol{m}}_{\mathcal{S}_{k}}$ is the unscaled
correction. For any candidate quantity $\boldsymbol{x}_{i}$, define
\[
\mathcal{Z}_{\mathcal{C}_{k}}(\boldsymbol{x}_{i})=\frac{\boldsymbol{x}_{i}-\boldsymbol{\mu}_{\mathcal{C}_{k}}}{\max\{\boldsymbol{\sigma}_{\mathcal{C}_{k}},10^{-6}\}},\tag{C.10}
\]
where $\boldsymbol{\mu}_{\mathcal{C}_{k}}$ and $\boldsymbol{\sigma}_{\mathcal{C}_{k}}$
are the mean and standard deviation over $i\in\mathcal{C}_{k}$. The
same normalization of the analytic anchor is used during training
and deployment. We have the normalized analytic utility:
\[
\bar{U}_{i}=\mathcal{Z}_{\mathcal{C}_{k}}\left(\log\max\{U_{i},10^{-12}\}\right),\hspace*{0.3cm}b_{i}=\mathcal{Z}_{\mathcal{C}_{k}}\left(\widehat{\Delta}_{i}\right).\tag{C.11}
\]
The calibrated score and training teacher are given by $\widetilde{U}_{i}=\bar{U}_{i}+\beta_{\mathrm{acq}}\tanh\left(r_{\psi}(i;\mathbf{z}_{\mathcal{S}_{k}})\right)$
and $h_{i}=\lambda_{m}\bar{U}_{i}+(1-\lambda_{m})b_{i}.$ The corresponding
candidate distributions are
\[
\ensuremath{p^{\mathrm{\tau}}_{i}=\frac{\exp(h_{i}/T_{\mathrm{\tau}})}{{\displaystyle \sum_{j\in\mathcal{C}_{k}}\exp(h_{j}/T_{\mathrm{\tau}})}},\hspace*{0.3cm}}\ensuremath{p^{\psi}_{i}=\frac{\exp(\widetilde{U}_{i}/T_{\psi})}{{\displaystyle \sum_{j\in\mathcal{C}_{k}}\exp(\widetilde{U}_{j}/T_{\psi})}},\ensuremath{\hspace*{0.3cm}p^{\mathrm{a}}_{i}=\frac{\exp(\bar{U}_{i}/T_{\psi})}{{\displaystyle \sum_{j\in\mathcal{C}_{k}}\exp(\bar{U}_{j}/T_{\psi})}},}}\tag{C.12}
\]
where $T_{\mathrm{\tau}},T_{\psi}>0$ are temperatures. These distributions
are used for training, and deployment selects locations from scores.
Let $i^{\star}=\arg\max_{i}h_{i}$, $\widehat{i}=\arg\max_{i}\widetilde{U}_{i}$,
$h_{\max}=\max_{i}h_{i}$, and $h_{\min}=\min_{i}h_{i}$. The distillation
and selection losses are
\[
\mathcal{L}_{\mathrm{dist}}=-\sum_{i}p^{\mathrm{\tau}}_{i}\log p^{\psi}_{i},\hspace*{0.1cm}\mathcal{L}_{\mathrm{top}}=-\log\frac{\exp(\widetilde{U}_{i^{\star}})}{{\displaystyle \sum_{j}\exp(\widetilde{U}_{j})}},\hspace*{0.1cm}\mathcal{L}_{\mathrm{reg}}=\frac{\sum_{i}p^{\psi}_{i}\left(h_{\max}-h_{i}\right)}{\max(h_{\max}-h_{\min},10^{-8})}.\tag{C.13}
\]
They respectively transfer the teacher preference distribution, promote
its highest-scoring candidate, and penalize expected teacher regret.
We further supervise the learned departure from the normalized analytic
anchor using
\[
\mathcal{L}_{\mathrm{cal}}=\frac{1}{2}\sum_{i}(p^{\mathrm{\tau}}_{i}+p^{\mathrm{a}}_{i})\mathcal{H}\left(\widetilde{U}_{i}-\bar{U}_{i}-\Pi_{[-1,1]}(h_{i}-\bar{U}_{i})\right),\tag{C.14}
\]
where $\ensuremath{\mathcal{H}(d)=d^{2}/2}$ for $\ensuremath{|d|<1}$,
otherwise $|d|-1/2$. We additionally introduce
\[
\begin{aligned}\mathcal{L}_{\mathrm{mar}} & =\mathbf{1}_{\{\widehat{i}\ne i^{\star}\}}\mathrm{max}\left(0.5+\widetilde{U}_{\widehat{i}}-\widetilde{U}_{i^{\star}},0\right),\hspace*{0.3cm}\mathcal{L}_{\mathrm{anc}}=\mathrm{max}\left(\sum_{i}p^{\mathrm{a}}_{i}h_{i}-\sum_{i}p^{\psi}_{i}h_{i},0\right).\end{aligned}
\tag{C.15}
\]
The margin term promotes separation from an incorrect greedy winner,
while the anchor-relative term prevents the learned policy from having
lower teacher-expected value than the analytic-anchor distribution.
All sums and extrema above are over $\mathcal{C}_{k}$.

The complete objective is $\mathcal{L}_{\mathrm{scorer}}=\sum_{j\in\mathcal{J}}\lambda_{j}\mathcal{L}_{j},$
where $\mathcal{J}=(\mathrm{dist,top,reg,cal,mar,anc})$ specifies
the ordered set of loss terms. Only the scorer $\psi$ is optimized;
prior generation, graph conditioning, analytic scoring, and teacher
construction are detached. Following the order in $\mathcal{J}$,
we use $\lambda_{\mathcal{J}}=(1,0.1,2,0.1,0.1,0)$, $(T_{\mathrm{\tau}},T_{\psi})=(0.15,0.30)$
and $\lambda_{m}=0$ for RFR, and $\lambda_{\mathcal{J}}=(1,0.1,0.5,0.25,0,0.05)$,
$(T_{\mathrm{\tau}},T_{\psi})=(0.25,0.25)$ and $\lambda_{m}=0.2$
for DC. We set the refinement strength to $\rho_{g}=1.4$ for RFR
and $\rho_{g}=1.0$ for DC, with a base residual-score scale of $\beta_{\mathrm{acq}}=2.0$
for both tasks.

\end{document}